\documentclass[11pt]{article}

\usepackage{graphicx}  
\usepackage{enumitem}
\usepackage{xcolor}
\usepackage{mathtools}
\usepackage{graphicx}
\usepackage{amsmath} 
\usepackage{amssymb}
\usepackage{amsthm}
\usepackage{array}
\usepackage[margin=1in]{geometry}
\usepackage{natbib}

\usepackage[ruled,linesnumbered,vlined]{algorithm2e}
\usepackage[capitalize]{cleveref}

\usepackage[mathcal]{euscript} 

\usepackage{titlesec}

\crefname{algocf}{Algorithm}{Algorithm}
\Crefname{algocf}{Algorithm}{Algorithm}

\titlespacing{\paragraph}{0pt}{0pt}{1ex}

\DeclarePairedDelimiter{\norm}{\lVert}{\rVert}
\DeclarePairedDelimiter{\abs}{\vert}{\rvert}

\newcommand{\citeasnoun}[1]{\citeauthor{#1}~(\citeyear{#1})}
\newcommand{\E}{\mathbb{E}}
\renewcommand{\P}{\mathbb{P}}
\newcommand{\opt}{^\star}
\newenvironment{mprog}{\begin{array}{>{\displaystyle}l>{\displaystyle}l>{\displaystyle}l}}{\end{array}}
\newcommand{\stc}{\\[1ex]  \mbox{s.t.} &}
\newcommand{\cs}{\\[1ex] & }
\newcommand{\minimize}[1]{\min_{#1} &}

\newcommand{\tr}{^{\mathsf{T}}}
\newcommand{\one}{\mathbf{1}}
\newcommand{\Real}{\mathbb{R}}
\renewcommand{\ss}{\,:\,}

\newcommand{\eye}{\mathbf{I}}
\newcommand{\zero}{\mathbf{0}}

\usepackage{natbib}

\newcommand{\todo}[1]{}

\newcommand{\states}{\mathcal{S}}
\newcommand{\actions}{\mathcal{A}}

\newcommand{\aset}{\mathcal{P}}
\newcommand{\aseth}{\aset^{H}}
\newcommand{\asetht}{\aset^{M}}
\newcommand{\asetb}{\aset^{B}}

\newcommand{\vset}{\mathcal{V}}

\newcommand{\dataset}{\mathcal{D}}
\newcommand{\BU}{T}

\newcommand{\RBU}{\widehat{T}}

\newcommand{\statecount}{S}

\theoremstyle{plain}
\newtheorem{theorem}{Theorem}[section]

\newtheorem{lemma}[theorem]{Lemma}
\newtheorem{proposition}[theorem]{Proposition}
\theoremstyle{definition}
\newtheorem{definition}[theorem]{Definition}
\theoremstyle{remark}
\newtheorem{remark}{Remark}[section]
\newtheorem{example}{Example}[section]

\author{
Marek Petrik,  
Reazul Hasan Russell\thanks{Equal contribution} \\
Department of Computer Science,\\ University of New Hampshire
}

\date{}

\title{Beyond Confidence Regions: Tight Bayesian Ambiguity Sets\\for Robust MDPs}

\begin{document}
	
	\maketitle
	
	
	
	\begin{abstract}
		Robust MDPs~(RMDPs) can be used to compute policies with provable worst-case guarantees in reinforcement learning. The quality and robustness of an RMDP solution are determined by the ambiguity set---the set of plausible transition probabilities---which is usually constructed as a multi-dimensional confidence region. Existing methods construct ambiguity sets as confidence regions using concentration inequalities which leads to overly conservative solutions. This paper proposes a new paradigm that can achieve better solutions with the same robustness guarantees without using confidence regions as ambiguity sets. To incorporate prior knowledge, our algorithms optimize the size and position of ambiguity sets using Bayesian inference. Our theoretical analysis shows the safety of the proposed method, and the empirical results demonstrate its practical promise. 
	\end{abstract}
	
	\section{Introduction}
	
	Markov decision processes (MDPs) provide a versatile framework for modeling reinforcement learning problems~\citep{Bertsekas1996,Sutton1998,Puterman2005}.  An important limitation of MDPs is that they assume that transition probabilities and rewards are known exactly which is rarely the case. Limited data sets, modeling errors, value function approximation, and noisy data are common reasons for errors in transition probabilities~\citep{Iyengar2005,Wiesemann2013,Petrik2014}. This results in policies that are brittle and fail in real-world deployments.
	
	This work targets \emph{batch} reinforcement learning~\citep{Lange2012} in which a good policy needs to be computed from a logged dataset without interacting with a simulator. This setting is common when experimentation is either too expensive or time-consuming, such as in medical care, agriculture, or even robotics. 
	
	Batch reinforcement learning introduces two important challenges~\citep{Petrik2016a,Thomas2015,Li2015,Jiang2015b}. First, the amount of data may be insufficient to compute a good policy. Second, evaluating the quality of a policy without simulation can be difficult. We tackle these challenges by computing a \emph{robust} policy and a high-confidence lower bound on its \emph{true} return. A lower bound on the return can prevent the deployment of a bad policy or justify the need for more data or better modeling~\citep{Petrik2016a,Lim2013,Hanasusanto2013}. 
	
	Robust MDPs (RMDPs) are a convenient model for computing reinforcement learning policies with strong worst-case guarantees. They generalize MDPs by assuming that transition probabilities and/or rewards are not known precisely. They can, instead, take on any value from a so-called \emph{ambiguity set} (also known as an uncertainty set) which represents a set of plausible transition probabilities~\citep{Xu2006,Xu2009,Mannor2012,Petrik2012,Hanasusanto2013,Tamar2014a,Delgado2016,Petrik2016a}. RMDPs are reminiscent of dynamic zero-sum games: the decision maker chooses the best actions, while the adversarial nature chooses the worst transition probabilities from the ambiguity set.
	
	The quality of the optimal RMDP policy depends on the ambiguity set used. It must be the smallest set that is large enough to guarantee that the solution is a lower bound. RL algorithms usually construct data-driven ambiguity sets as \emph{confidence regions} derived from concentration inequalities~\citep{Weissman2003xx,Auer2010a,Thomas2015,Petrik2016a}. Using, for example, a 95\% confidence region over possible transition probabilities translates to a 95\% confidence that the RMDP return lowers the true one. Unfortunately, concentration inequalities lead to solutions that are too conservative to be practical. Another approach is to construct ambiguity sets from likelihood levels of probability distributions, but this method requires complex modeling and does not provide finite-sample guarantees~\citep{Iyengar2005,Nilim2005,Ben-Tal2009,Bertsimas2017}. 
	\todo{Mention how the ambiguity sets are adapated to value functions.}
	In this paper, we argue that constructing ambiguity sets as confidence regions leads to solutions that are unnecessarily conservative. Confidence regions inherently provide robust guarantees for \emph{all} policies and \emph{all} value functions \emph{simultaneously}. It is sufficient, instead, to provide the guarantees for the optimal RMDP policy and value function. Our algorithm (RSVF) provides a tighter lower bound on the return of the optimal policy by interleaving RMDP computations with optimizing the \emph{size} and the \emph{position} of ambiguity sets. Using (hierarchical) Bayesian models helps to further tighten the lower bounds by leveraging prior domain knowledge. We also derive new $L_1$ concentration inequalities of possible independent interest. 
	
	\citeasnoun{Gupta2015} also constructs ambiguity sets that are not confidence regions. However, their setting and objectives are markedly different from ours and do not readily apply to RMDPs. In general, Bayesian methods for constructing ambiguity sets for RMDPs are not yet understood well and have received only limited attention~\citep{Xu2009}.
	
	Confidence regions derived from concentration inequalities have been used previously to compute bounds on the true return in off-policy policy evaluation~\citep{Thomas2015,Thomas2016}. These methods, unfortunately, do not readily generalize to the policy optimization setting, which we target. Other work has focused reducing variance rather than on high-probability bounds~\citep{Munos2016,Li2015,Jiang2015b}. Methods for exploration in reinforcement learning, such as MBIE or UCRL2, also construct ambiguity sets using concentration inequalities~\citep{Strehl2008a,Auer2010,Taleghan2015,Dietterich2013,Strehl2008a} and compute optimistic (upper) bounds to guide exploration. 
	
	The remainder of the paper is organized as follows. \cref{sec:robust_mdps} formally describes the framework and goals of the paper. \cref{sec:confidence_interval} outlines new and existing methods for building ambiguity sets as frequentist confidence regions or Bayesian credible sets. The methods construct these sets around the most-probable transition probabilities. \cref{sec:multiple} describes our main contribution, RSVF, a new method for constructing tight ambiguity sets from Bayesian models that are adapted to the optimal policy. RSVF provides tighter robustness guarantees without using confidence regions which we justify in \cref{sec:why}. Finally, \cref{sec:experiments} presents empirical results on several problem domains. 
	
	\section{Problem Statement: Data-driven RMDPs} \label{sec:robust_mdps}
	
	This section formalizes our goals and reviews relevant results for robust Markov decision processes~(RMDPs). Throughout the paper, we use the symbol $\Delta^\statecount$ to denote the probability simplex in $\Real_+^\statecount$. The symbols $\one$ and $\zero$ denote vectors of all ones and zeros, respectively, of an appropriate size. The symbol $\eye$ represents the identity matrix.
	
	\subsection{Safe Return Estimate}
	
	The underlying reinforcement learning problem is a Markov decision process with states $\states = \{1, \ldots, S \}$ and actions $\actions = \{1, \ldots, A \}$. The rewards $r:\states\times\actions\to\Real$ are known but the true transition probabilities $P\opt: \states \times \actions \to \Delta^\states$ are unknown. The transition probability vector for a state $s$ and an action $a$ is denoted by $p\opt_{s,a}$. As this is a \emph{batch} reinforcement learning setting, a fixed dataset $\dataset$ of transition samples is provided: 
	$\dataset \subseteq \{ (s,a,s') \ss s,s'\in\states, a\in \actions \}$. The only assumption about $\dataset$ is that the state $s'$ in $(s,a,s') \in \mathcal{S}$ is distributed according to the \emph{true} transition probabilities: $s' \sim P\opt(s,a,\cdot)$. We make no assumptions on the policy used to generate the dataset. 
	
	We assume the standard $\gamma$-discounted infinite horizon objective~\citep{Puterman2005}. Because this paper analyzes the impact of using different transition probabilities, we use a subscript to indicate which ones are used. The optimal value function for some transition probabilities $P$ is, therefore, denoted as $v\opt_P: \states \rightarrow \Real$, and the value function for a \emph{deterministic policy} $\pi: \states \rightarrow \actions$ is denoted as $v_P^\pi$. The set of all deterministic stationary policies is denoted by $\Pi$. The total return $\rho(\pi,P)$ of a policy $\pi$ under transition probabilities $P$ is:
	\[\rho(\pi,P) = p_0\tr v^\pi_P, \] 
	where $p_0$ is the initial distribution.
	
	Our \emph{objective} is to compute a policy $\pi: \states\to\actions$ that maximizes the return $\rho(\pi,P\opt)$. Because the objective depends on the unknown $P\opt$, we instead compute a policy with the greatest lower guarantee on the return. The term \emph{safe return estimate} refers to the lower bound estimate.
	\begin{definition}[Safe Return Estimate] \label{def:safety}
		The estimate $\tilde{\rho}: \Pi \rightarrow \Real$ of return is called \emph{safe} for a policy $\pi$ with probability $1-\delta$ if it satisfies:
		\[ 
		\P_{P\opt} \Bigl[ \tilde{\rho}(\pi) \le \rho(\pi,P\opt) \;\vline\; \dataset \Bigr] \ge 1-\delta~.
		\] 
	\end{definition}
	
	\begin{remark}
		Under Bayesian assumptions, $P\opt$ is a random variable and the guarantees are conditional on the dataset $\dataset$. This is different from the frequentist approach, in which the random variable is $\dataset$ and the guarantees are conditional on $P\opt$. See, for example, Sections 5.2.2 and 6.1.1 in \citeasnoun{Murphy2012} for a discussion of the merits of the two approaches. Unless it is apparent from the context, we indicate whether the probability is conditional on $\dataset$ or $P\opt$.
	\end{remark}
	
	Having a safe return estimate is very important in practice. A low safe estimate informs the stakeholders that the policy may not perform well when deployed. They may, instead, choose to gather more data, keep the existing (baseline) policy, or use a more informative domain~\citep{Petrik2016a,Laroche2017}.
	
	\subsection{Robust MDPs} 
	
	Robust Markov Decision Processes (RMDPs) are a convenient model that can be used to compute and tractably optimize the \emph{safe} return estimate ($\max_{\pi} \tilde{\rho}(\pi)$). Our RMDP model has the same states $\states$, actions $\actions$, rewards $r_{s,a}$  as the MDP. The transition probabilities for each state $s$ and action $a$, denoted as $p_{s,a} \in \Delta^\statecount$, are assumed chosen adversarialy from an \emph{ambiguity set} $\aset_{s,a}$. We use $\aset$ to refer cumulatively to $\aset_{s,a}$ for all states $s$ and actions $a$. 
	
	We restrict our attention to \emph{compact} and so-called $s,a$-rectangular ambiguity sets. Rectangular ambiguity sets allow the nature to choose the worst transition probability independently for each state and action~\citep{LeTallec2007,Wiesemann2013}. Limitations of rectangular ambiguity sets are well known~\citep{Mannor2016,Tirinzoni2018,Goyal2018} but they represent a simple, tractable, and practical model. A convenient way of defining ambiguity sets is to use a norm-distance from a given \emph{nominal transition probability} $\bar{p}_{s,a}$:
	\begin{equation} \label{eq:ambiguity_set}
	\aset_{s,a} = \bigl\{p \in \Delta^\statecount \ss \norm{p - \bar{p}_{s,a} }_1 \le \psi_{s,a} \bigr\}
	\end{equation}
	for a given $\psi_{s,a}\ge 0$ and a nominal point $\bar{p}_{s,a}$. We focus on ambiguity sets defined by the $L_1$ norm because they give rise to RMDPs that can be solved very efficiently~\citep{Ho2018}. 
	
	RMDPs have properties that are similar to regular MDPs~(see, for example, \citep{Bagnell2001b,Kalyanasundaram2002,Nilim2005,LeTallec2007,Wiesemann2013}). The robust Bellman operator $\RBU_\aset$ for an ambiguity set $\aset$ for a state $s$ computes the best action with respect to the worst-case realization of the transition probabilities:
	\begin{equation} \label{eq:bellman_definition}
	(\RBU_\aset v)(s) := \max_{a\in\actions}\min_{p \in\aset_{s,a}}  (r_{s,a} + \gamma \cdot p\tr v)  
	\end{equation} 
	The symbol $\RBU^\pi_\aset$ denotes a robust Bellman update for a given \emph{stationary} policy $\pi$. The optimal robust value function $\hat{v}\opt$, and the robust value function $\hat{v}^\pi$ for a policy $\pi$ must, similarly to MDPs, satisfy: 
	\[ \hat{v}\opt = \RBU_\aset \hat{v}\opt, \qquad \hat{v}^\pi = \RBU_\aset^\pi \hat{v}^\pi  ~.\]
	In general, we use a hat to denote quantities in the RMDP and omit it for the MDP. When the ambiguity set $\aset$ is not obvious from the context, we use it as a subscript $\hat{v}\opt_\aset$. The robust return $\hat{p}$ is defined as~\citep{Iyengar2005}:
	\[ \hat{\rho}(\pi, \aset) = \min_{P\in\aset} \rho(\pi, P) = p_0\tr \hat{v}^\pi_\aset~, \]
	where $p_0 \in \Delta^S$ is the initial distribution. In the remainder of the paper, we describe methods that construct $\aset$ from $\dataset$ in order to guarantee that $\hat{\rho}$ is a tight lower bound on $\rho$.

	\section{Ambiguity Sets as Confidence Regions} \label{sec:confidence_interval}
	
	In this section, we describe the standard approach to constructing ambiguity sets as multidimensional confidence regions and propose its extension to the Bayesian setting. This is a natural approach but, as we discuss later, may be unnecessarily conservative.
	
	Before describing how the ambiguity sets are constructed, we need the following auxiliary lemma. The lemma shows that when the robust Bellman update lower-bounds the true Bellman update then the value function estimate is safe.
	\begin{lemma} \label{prop:single_to_many}
		Consider a policy $\pi$, its robust value function $\hat{v}^\pi$, and true value function $v^\pi$ such that $\hat{v}^\pi = \RBU^\pi \hat{v}^\pi$ and $v^\pi = \BU^\pi v^\pi$. Then, $\hat{v}^\pi \le v^{\pi}$ element-wise whenever $\RBU^\pi \hat{v}^\pi \le \BU^{\pi} \hat{v}^\pi$.
	\end{lemma}
	\vspace{-0.2cm}\noindent The proof is deferred to \cref{app:proofs}. 
	
	Note that the inequality holds with respect to the robust value function $\hat{v}^\pi$. The requirement $\RBU^\pi \hat{v}^\pi \le \BU^{\pi} \hat{v}$ in \cref{prop:single_to_many} can be restated as:
	\begin{equation}  \label{eq:single_to_many}
	\min_{p \in \aset_{s,a}} p\tr \hat{v}^\pi \le p\tr_{s,a} \hat{v}^\pi  ~, 
	\end{equation}
	for each state $s$ and action $a = \pi(s)$. It can be readily seen that the inequality above is satisfied when $p_{s,a} \in \aset_{s,a}$. Next, we describe two algorithms for constructing ambiguity sets $\aset_{s,a}$ such that $p_{s,a}\opt \in \aset_{s,a}$ with high probability.
	
	\subsection{Distribution-free Confidence Region} 
	
	Distribution-free confidence regions are used widely in reinforcement learning to achieve robustness~\citep{Petrik2016a} and to guide exploration~\citep{Taleghan2015,Strehl2008}. The confidence region is constructed around the mean transition probability by combining the Hoeffding inequality with the union bound~\citep{Weissman2003xx,Petrik2016a}. We refer to this set as a \emph{Hoeffding confidence region} and define it as follows for each $s$ and $a$:
	\begin{equation*} 
	\aseth_{s,a} = \left\{ p\in\Delta^S : \norm{p - \bar{p}_{s,a} }_1 \le \sqrt{\frac{2}{n_{s,a}} \log \frac{S A 2^{S}}{\delta} } \right\},
	\end{equation*}
	where $\bar{p}_{s,a}$ is the mean transition probability computed from $\dataset$ and $n_{s,a}$ is the number of transitions in $\dataset$ originating from state $s$ and an action $a$.
	
	\begin{theorem} \label{cor:hoeffding_bound}
		The robust value function $\hat{v}_{\aseth}$ for the ambiguity set $\aseth$ satisfies: 
		\begin{equation} 
		\P_{\dataset} \left[ \hat{v}^\pi_{\aseth} \le v_{P\opt}^{\pi},  \; \forall \pi\in\Pi ~\middle|~ P\opt \right] \ge 1-\delta~.
		\end{equation}
		In addition, suppose that $\hat{\pi}\opt_{\aseth}$ is the optimal solution to the robust MDP. Then, $p_0\tr \hat{v}\opt_{\aseth}$ is a \emph{safe} return estimate of $\hat{\pi}\opt_{\aseth}$. 
	\end{theorem}
	\vspace{-0.2cm} \noindent The proof is deferred to \cref{app:proofs} and is a simple extension of prior results~\citep{Petrik2016a}. 
	
	To better understand the limitations of using concentration inequalities, we derive a new, and significantly tighter, ambiguity set. The size of $\aseth$ grows linearly with the number of states because of the $2^S$ term. This means that the size of $\dataset$ must scale about quadratically with the number of states to achieve the same confidence. We shrink the Hoeffding set by assuming that the value function is monotone (e.g. $v(1) \ge v(2) \ge \ldots$). It is then sufficient to use the following significantly smaller ambiguity set:
	\[
	\asetht_{s,a} = \left\{ p \in\Delta^S \ss \norm{p\opt - \bar{p}_{s,a} }_1 \le \sqrt{\frac{2}{n_{s,a}} \log \frac{S^2 A}{\delta} } \right\}~.
	\]
	Note the lack of the $2^S$ term in comparison with $\aseth$. This auxiliary result is proved in \cref{sec:improved_bounds}. We emphasize that the aim of this bound is to understand the limitations of distribution free bounds, and we use this set even the monotonicity is not assured.

	\subsection{Bayesian Credible Region (BCI)} 
	
	We now describe how to construct ambiguity sets from Bayesian credible (or confidence) regions. To the best of our knowledge, this approach has not been studied in depth previously. The construction starts with a (hierarchical) Bayesian model that can be used to sample from the posterior probability of $P\opt$ given data $\dataset$. The implementation of the Bayesian model is irrelevant as long as it generates posterior samples efficiently. For example, one may use a Dirichlet posterior, or use MCMC sampling libraries like JAGS, Stan, or others~\citep{Gelman2014}. 
	
	The posterior distribution is used to optimize for the \emph{smallest} ambiguity set around the mean transition probability. Smaller sets, for a fixed nominal point, are likely to result in less conservative robust estimates. The BCI ambiguity set is defined as follows:
	\[ \asetb_{s,a} = \left\{ p\in\Delta^S \ss \norm{p - \bar{p}_{s,a}}_1 \le \psi_{s,a}^B \right\} ~,\]
	where nominal point is $\bar{p}_{s,a} = \E_{P\opt}[p\opt_{s,a} ~|~ \dataset]$. 
	
	There is no closed-form expression for the Bayesian ambiguity set size. It must be computed by solving the following optimization problem for each state $s$ and action $a$:
	\[
	\psi^B_{s,a} = \min_{\psi\in\Real_+} \left\{\psi \ss \P\left[ \norm{p\opt_{s,a} - \bar{p}_{s,a}}_1 > \psi ~|~ \dataset \right] < \frac{\delta}{SA} \right\}~.
	\]
	The nominal point $\bar{p}_{s,a}$ is fixed (not optimized) to preserve tractability. This optimization problem can be solved by the Sample Average Approximation~(SAA) algorithm~\citep{Shapiro2014}. The main idea is to sample from the posterior distribution and then choose the minimal size $\psi_{s,a}$ that satisfies the constraint. \Cref{alg:bayes}, in the appendix, summarizes the sort-based method.  
	
	We assume that it is possible to draw enough samples from $P\opt$ that the sampling error becomes negligible. Because the finite-sample analysis of SAA is simple but tedious, we omit it in the interest of clarity. 
	
	The Bayesian ambiguity sets also guarantee safe estimates.
	\begin{theorem} \label{cor:bci_bound}
		The robust value function $\hat{v}_{\asetb}$ for the ambiguity set $\asetb$ satisfies:
		\[ \P_{P\opt} \left[ \hat{v}^\pi_{\asetb} \le v_{P\opt}^{\pi}, \; \forall \pi\in\Pi ~\middle|~ \dataset \right] \ge 1-\delta~. \]
		In addition, suppose that $\hat{\pi}\opt_{\asetb}$ is the optimal solution to the robust MDP. Then, $p_0\tr \hat{v}\opt_{\asetb}$ is a \emph{safe} return estimate of $\hat{\pi}\opt_{\asetb}$. 
	\end{theorem}
	\vspace{-0.2cm} \noindent The proof is deferred to \cref{app:proofs}.
	
	\begin{figure}
		\centering
		\includegraphics[width=0.39\linewidth]{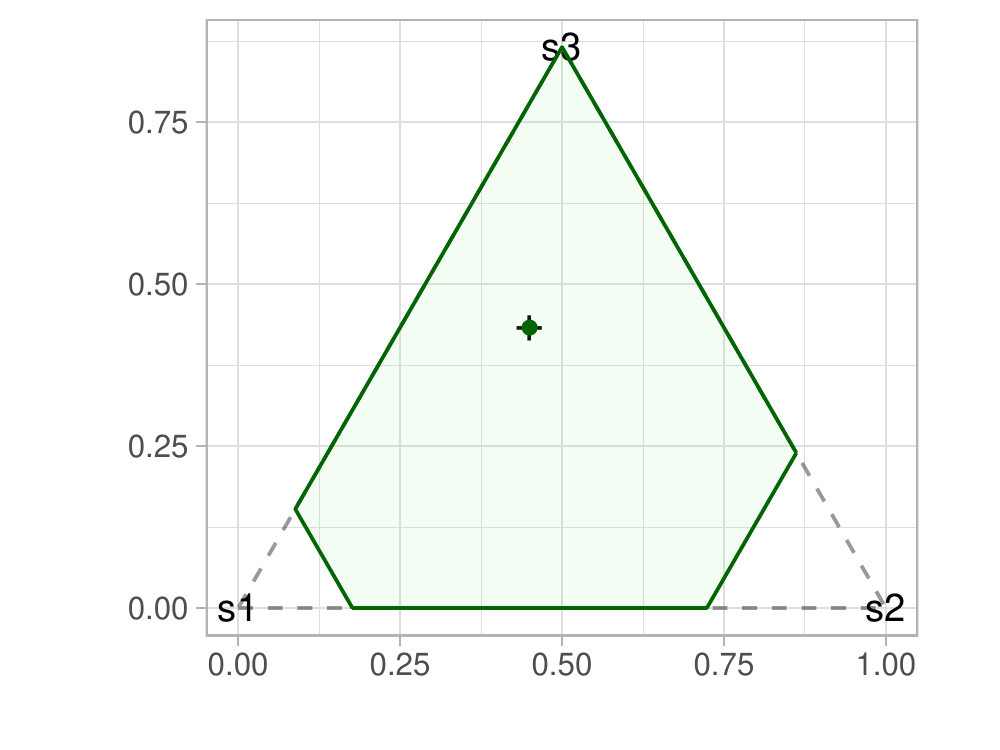}    
		\includegraphics[width=0.39\linewidth]{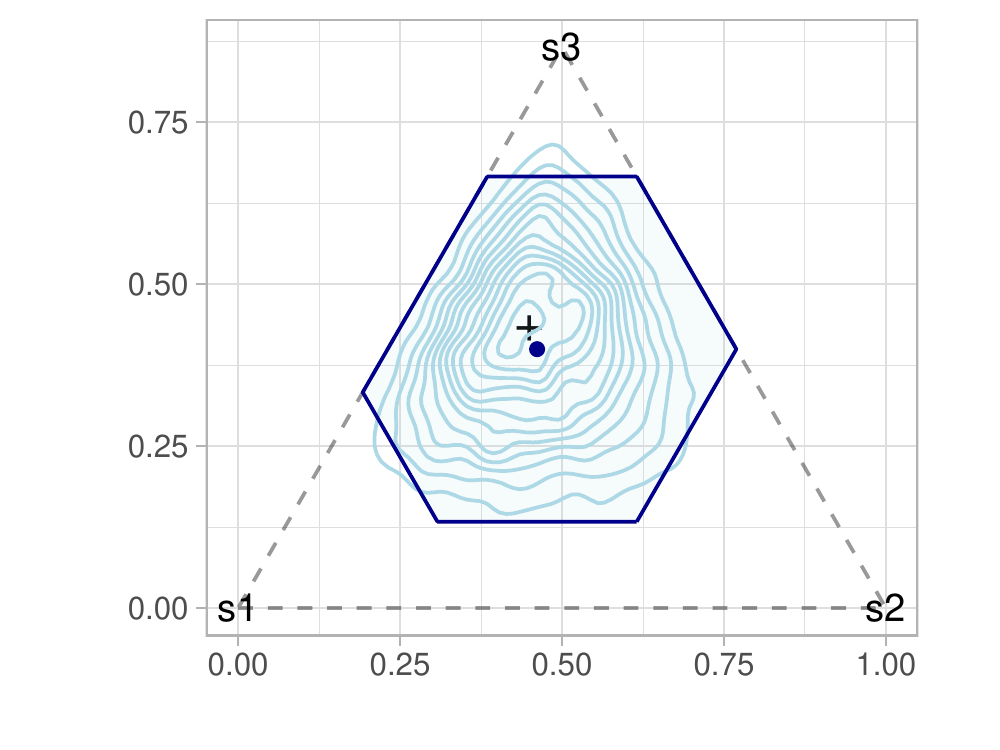}    
		\caption{$90\%$-confidence ambiguity sets $\aseth_{s_1,a_1}$ (left) and $\asetb_{s_1,a_1}$ (right)  projected onto the 3-dimensional probability simplex. } \label{fig:confidence_example}
	\end{figure}
	
	\begin{example} \label{exm:basic} 
		Consider an RMDP with 3 states: $s_1, s_2, s_3$ and a single action $a_1$. Assume that the true transition probability is $P\opt(s_1,a_1,\cdot) = [0.3, 0.2, 0.5]$. In $\dataset$, there are $3$ occurrences of transitions $(s_1,a_1,s_1)$, $2$ of transitions $(s_1,a_1,s_2)$, and $5$ of transitions $(s_1,a_1,s_3)$. The prior distribution over $p\opt_{s_1,a_1}$ is Dirichlet with concentration parameters $\alpha = (1,1,1)$. \Cref{fig:confidence_example} depicts ambiguity sets for state $s_1$ and action $a_1$. The plus sign marks $p\opt_{s_1,a_1}$, while the dot marks the nominal point of the ambiguity set; the contours indicate the density of the posterior Dirichlet distribution.
	\end{example}
	
	BCI ambiguity sets $\asetb$ can be much less conservative than Hoeffding sets $\aseth,\asetht$ given informative priors, but also involve greater computation complexity. Next, we further improve on BCI.
	
	\section{Optimizing Ambiguity Sets to Values} \label{sec:multiple}
	
	In this section, we describe a new algorithm for constructing Bayesian ambiguity sets that can compute less-conservative lower bounds on the return.  RSVF (robustification with sensible value functions) is a Bayesian method that uses posterior samples like BCI. The main difference is that RSVF interleaves solving the robust MDP with constructing ambiguity sets. This means that it can construct sets that are better adapted to the optimal policy. 
	
	RSVF is outlined in \cref{alg:IAVF}. It intends to construct an \emph{optimal ambiguity set} $\aset$ for the \emph{optimal robust} value function $\hat{v}\opt_\aset$. This approach, of course, creates a difficult dependency loop. The value function $\hat{v}\opt_\aset$ depends on the ambiguity set $\aset$ and the optimal set $\aset$ depends on $\hat{v}\opt_\aset$. RSVF takes an optimistic (and heuristic) approach to this hurdle. It starts with a small set of potential optimal value functions (POV) and constructs an ambiguity set that is safe for these value functions. It keeps increasing the POV set until $\hat{v}\opt$ is in the set and the policy is safe. 
	
	\begin{algorithm}
		\KwIn{Confidence $1-\delta$ and posterior $\P_{P\opt}[\cdot ~|~\dataset]$ }
		\KwOut{Policy $\pi$ and lower bound $\tilde{\rho}(\pi)$}
		$k\gets 0$\;
		Pick some initial value function $\hat{v}_0$\;
		Initialize POV: $\vset_0 \gets \emptyset$ \;    
		\Repeat{safe for all $s,a$: $\mathcal{K}_{s,a}(\hat{v}_{k}) \cap \aset_{s,a}^{k} \neq \emptyset$}{
			Augment POV: $\vset_{k+1} \gets \vset_k \cup \{ v_k \}$ \;
			For all $s,a$ update  $\aset^{k+1}_{s,a} \gets \mathcal{L}_{s,a}(\vset_{k+1})$ \;
			Solve $\hat{v}_{k+1} \gets \hat{v}\opt_{\aset_{k+1}}$ and $\hat{\pi}_{k+1}\gets \hat{\pi}\opt_{\aset_{k+1}}$\;
			$k \gets k + 1$ \;
		}
		\Return $(\hat{\pi}_k, p_0\tr \hat{v}_k)$ \;
		\caption{RSVF: Adapted Ambiguity Sets} \label{alg:IAVF}
	\end{algorithm}
	
	We are now ready to describe how the ambiguity sets in \cref{alg:IAVF} are constructed. The set $\mathcal{K}_{s,a}(v)$, for each $s,a$, denotes the set of safety-sufficient transition probabilities. That means that if the ambiguity set $\aset_{s,a}$ intersects $\mathcal{K}_{s,a}(\hat{v}^\pi_\aset)$ for each state $s$ and action $a$ then the value function $\hat{v}^\pi_\aset$ is safe. This set is defined as follows:
	\begin{equation} \label{eq:optimal_hyperplane}
	\begin{aligned}
	\mathcal{K}_{s,a}(v) &= \left\{ p\in\Delta^S \ss p\tr v \le g_{s,a}(v) \right\} \\
	g_{s,a}(v) &= \max \left\{ g \ss \P_{P\opt} [g \le (p\opt_{s,a})\tr v \;|\; \dataset] \ge \zeta \right\}~,
	\end{aligned}
	\end{equation}
	where $\zeta = 1 - \delta/(SA)$. The maximization in \eqref{eq:optimal_hyperplane} can be solved by SAA in time that is quasi-linear in the number of samples~\citep{Shapiro2014} as follows. Sample points $q_i$ from the probability distribution of $p\opt_{s,a}$ and sort them by $v\tr q_i$. The value $g_{s,a}$ is then the $1 - \delta/(SA)$ quantile. 
	
	The next lemma formalizes the safety-sufficiency of $\mathcal{K}$. Note that the rewards $r_{s,a}$ are not a factor in this lemma because they are certain and cancel out.
	\begin{lemma} \label{lem:optimal_set}
		Consider any ambiguity set $\aset_{s,a}$ and a value function $v$. Then $\min_{p\in\aset_{s,a}} p\tr v \le (p\opt_{s,a})\tr v$ with probability $1 - \delta/(SA)$ if and only if $\aset_{s,a} \cap \mathcal{K}_{s,a}(v) \neq \emptyset$.
	\end{lemma}
	
	\begin{proof}
		To show the ``if'' direction, let $\hat{p} \in \aset_{s,a} \cap \mathcal{K}_{s,a}(v)$. Such $\hat{p}$ exists because the intersection is nonempty. Then, $\min_{p\in\aset_{s,a}} p\tr v \le \hat{p}\tr v \le g_{s,a}(v)$. By definition, $g_{s,a}(v) \le (p\opt_{s,a})\tr v$ with probability $1 - \delta/(SA)$.
		
		To show the ``only if'' direction, suppose that $\hat{p}$ is a minimizer in $\min_{p\in\aset_{s,a}} p\tr v$. The premise translates to $\P_{P\opt} [ \hat{p}\tr v \le (p\opt_{s,a})\tr v \;|\; \dataset] \ge 1 - \delta/(SA)$. Therefore, $g_{s,a}(v) \ge \hat{p}\tr v$ and $\hat{p}\in \aset_{s,a} \cap \mathcal{K}_{s,a}$ and the intersection is non-empty.
	\end{proof}
	
	The purpose of the ambiguity set $\mathcal{L}_{s,a}(\vset)$ for POV set $\vset$ is to guarantee that the robust estimate for $s,a$ is safe for any of the value functions $v$ in $\vset$. Its center is chosen to minimize its size while intersecting $\mathcal{K}_{s,a}(v)$ for each $v$ in $\vset$ and is constructed as follows.
	\begin{equation} \label{eq:center_point}
	\begin{aligned}
	\mathcal{L}_{s,a}(\vset) &= \bigl\{ p \in \Delta^S \ss \norm{p - \theta_{s,a}(\vset) }_1 \le \psi_{s,a}(\vset) \bigr\} \\
	\psi_{s,a}(\vset) &= \min_{p\in\Delta^S} f(p), \quad \theta_{s,a}(\mathcal{V}) \in \arg\min_{p\in\Delta^S} f(p) \\
	f(p) &= \max_{v\in\vset} \min_{q \in \mathcal{K}_{s,a}(v)} \norm{q - p}_1 \\
	\end{aligned}
	\end{equation}
	The optimization in \eqref{eq:center_point} can be readily represented and solved as a linear program. The following lemma formalizes the properties of $\mathcal{L}_{s,a}$.
	\begin{lemma}
		For any finite set $\vset$ of value functions, the following inequality holds for all $v\in\vset$ simultaneously:
		\[ \P_{P\opt} \left[ \min_{p \in \mathcal{L}_{s,a}(\vset)} p\tr v \le (p_{s,a}\opt)\tr v ~\middle|~ \dataset \right] \ge 1-\frac{\delta}{SA} ~.\]
	\end{lemma}
	\begin{proof}
		Assume an arbitrary $v\in\vset$ and let $q\opt_v \in\arg\min_{q \in \mathcal{K}_{s,a}(v)} \norm{q - \theta_{s,a}(\vset)}_1$ using the notation of \eqref{eq:center_point}. From the definition of $\theta_{s,a}(\vset)$ in \eqref{eq:center_point}, the value $q_v$ is in the ambiguity set $\mathcal{L}_{s,a}(\vset)$. Given that also $q_v\in\mathcal{K}_{s,a}(v)$, \cref{lem:optimal_set} shows that:
		\[ \P_{P\opt} \left[ \min_{p \in \mathcal{L}_{s,a}(\vset)} p\tr v \le (p_{s,a}\opt)\tr v ~\middle|~ \dataset \right] \ge 1-\frac{\delta}{SA}~, \]
		because $q_v \in \mathcal{L}_{s,a}(v)\cup\mathcal{K}_{s,a}(v) \neq \emptyset$. This completes the proof since $v$ is any from $\vset$.
	\end{proof}
	
	\begin{figure}
		\centering
		\includegraphics[width=0.45\linewidth]{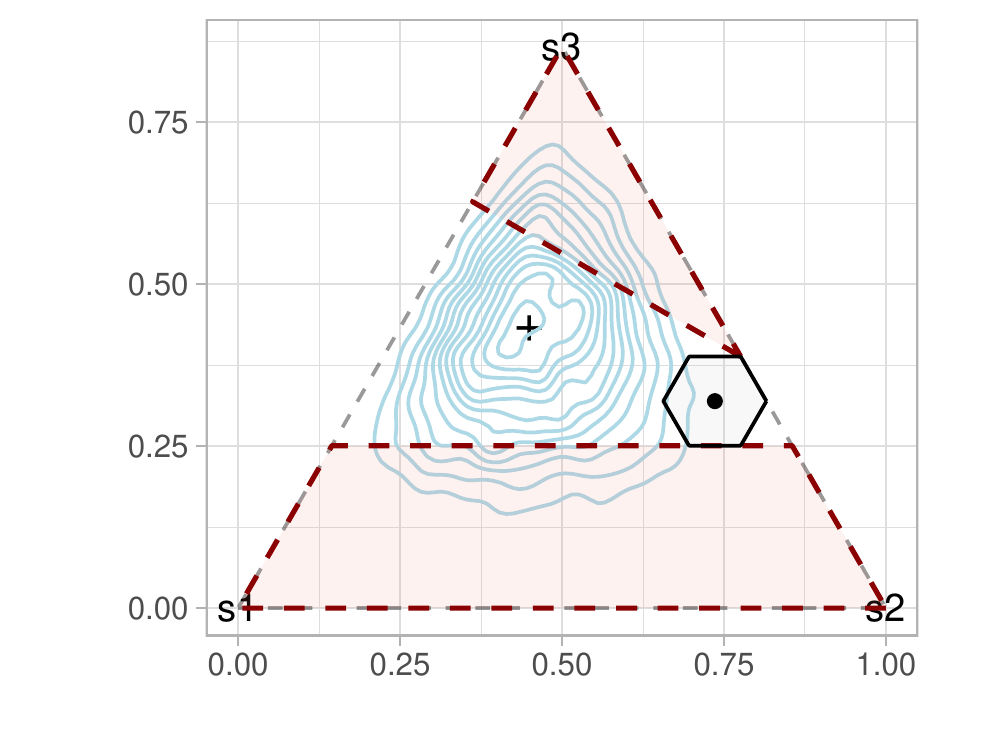}
		\caption{Simplex projection of sets $\mathcal{K}_{s_1,a_1}(v_i)$ (dashed red) for $i=1,2$ and $\mathcal{L}_{s_1,a_1}(\{v_1,v_2\})$ (solid black).}
		\label{fig:rsvf_set}
	\end{figure} 
	
	\begin{example}
		Assume the setting from \cref{exm:basic} and two value functions $v_1 = [0,0,1]$ and $v_2 = [2,1,0]$. \Cref{fig:rsvf_set} depicts $\mathcal{K}_{s_1,a_1}(v_1)$ and $\mathcal{K}_{s_1,a_1}(v_2)$ as dashed red and  $\mathcal{K}_{s_1,a_1}(\{v_1,v_2\})$ as solid black. 
	\end{example}
	
	We can now prove the safety of RSVF. 
	\begin{theorem} \label{thm:safety_condition_correct}
		Suppose that \cref{alg:IAVF} terminates with a policy $\hat\pi_k$ and a value function $\hat{v}_k$ in the iteration $k$. Then, the return estimate $p_0\tr\hat{v}_k$ is safe:
		\[ \P_{P\opt} \left[ p_0\tr\hat{v}_k \le p_0\tr v_{P\opt}^{\hat\pi_k} ~\middle|~ \dataset \right] \ge 1-\delta~. \]
	\end{theorem}
	\begin{proof}
		Recall that \cref{alg:IAVF} terminates only if $\mathcal{K}_{s,a}(\hat{v}_{k}) \cap \aset_{s,a}^{k} \neq \emptyset$
		for each state $s$ and action $a$. Then, according to  \cref{lem:optimal_set}, we get with probability $1 - \delta/(SA)$:
		\[ \min_{p\in\aset_{s,a}^{k}} p\tr \hat{v}_k \le (p_{s,a}\opt)\tr \hat{v}_k  \]
		for any fixed state $s$ and action $a$. By the union bound, the inequality holds simultaneously for all states and actions with probability $1-\delta$. That means that with probability $1-\delta$ we can derive the following using basic algebra:
		\begin{align*}
		\min_{p\in\aset_{s,a}^{k}} p\tr \hat{v}_k &\le (p_{s,a}\opt)\tr \hat{v}_k  & \forall s\in\states, a\in\actions \\
		r_{s,a} + \min_{p\in\aset_{s,a}^{k}} p\tr \hat{v}_k &\le r_{s,a} + (p_{s,a}\opt)\tr \hat{v}_k  & \forall s\in\states, a\in\actions \\
		\RBU_{\aset^k}^{\hat{\pi}_k} \hat{v}_k &\le \BU_{P\opt}^{\hat{\pi}_k} \hat{v}_k
		\end{align*}
		Note that $\hat{v}_k$ is the robust value function for the policy $\hat{\pi}_k$ since $\hat{v}_k = \hat{v}_{\aset_k}\opt$ and $\hat{\pi}_k = \hat{\pi}_{\aset_k}\opt$. \Cref{prop:single_to_many} finally implies that $\hat{v}_k \le v^{\hat{\pi}_k}_{P\opt}$ with probability $1-\delta$.
	\end{proof}
	
	RSVF, as described in \cref{alg:IAVF}, is not guaranteed to terminate. To terminate after a specific number of iterations, the algorithm can simply fall back to the BCI sets for states and actions for which the termination condition is not satisfied. We suspect that numerous other improvements to the algorithm are possible. 
	
	\section{Why Not Confidence Regions} \label{sec:why}
	
	Constructing ambiguity sets from confidence regions seems intuitive and natural. It may be surprising that RSVF abandons this intuitive approach. In this section, we describe two reasons why confidence regions are unnecessarily conservative compared to RSVF sets.
	
	The first reason why confidence regions are too conservative is because they assume that the value function depends on the true model $P\opt$. To see this, consider the setting of \cref{exm:basic} with $r_{s_1,a_1} = 0$. When an ambiguity set $\aset_{s_1,a_1}$ is built as a confidence region such that $\P[p\opt_{s_1,a_1} \in \aset_{s_1,a_1}] \ge 1-\delta$, it  satisfies:
	\begin{equation*} 
	\P_{P\opt} \left[ \min_{p \in \aset_{s,a}} p\tr v \le  (p_{s,a}\opt)\tr v, \; \forall v\in\Real^S  ~\middle|~ \dataset \right] \ge 1-\delta.
	\end{equation*}
	Notice the value function inside of the probability operator. \Cref{prop:single_to_many} shows that this guarantee is needlessly strong. It is, instead, sufficient that the inequality \eqref{eq:single_to_many} holds just for $\hat{v}^\pi$ which is independent of $P\opt$ in the Bayesian setting. The following weaker condition is sufficient to guarantee safety:
	\begin{equation} \label{eq:restricted_requirement}
	\P_{P\opt} \left[ \min_{p \in \aset_{s,a}} p\tr v \le  (p_{s,a}\opt)\tr v ~\middle|~ \dataset \right] \ge 1-\delta, \; \forall v\in\Real^S
	\end{equation}
	Notice that $v$ is outside of the probability operator. This set is smaller and provides the same guarantees, but may be more difficult to construct~\citep{Gupta2015}.
	
	The second reason why confidence regions are too conservative is because they construct a uniform lower bound for all policies $\pi$ as is apparent in \cref{cor:bci_bound}. This is unnecessary, again, as \cref{prop:single_to_many} shows. The robust Bellman update only needs to lower bound the Bellman update for the computed value function $\hat{v}^\pi$, not for all value functions. As a result, \eqref{eq:restricted_requirement}, can be further relaxed to:
	\begin{equation} \label{eq:rsvf_requirement}
	\P_{P\opt} \left[ \min_{p \in \aset_{s,a}} p\tr \hat{v}^{\pi_R} \le  (p_{s,a}\opt)\tr \hat{v}^{\pi_R} ~\middle|~ \dataset \right] \ge 1-\delta, 
	\end{equation}
	where $\pi_R$ is the optimal solution to the robust MDP. RSVF is less conservative because it constructs ambiguity sets that satisfy the weaker requirement of \eqref{eq:rsvf_requirement} rather than confidence regions. Deeper theoretical analysis of the benefits of using RSVF sets is very important but is beyond the scope of this work. Examples that show the benefits to be arbitrarily large or small can be constructed readily by properly choosing the priors over probability distributions.    
	
	\section{Empirical Evaluation} \label{sec:experiments}
	
	In this section, we empirically evaluate the safe estimates computed using Hoeffding, BCI, and RSVF ambiguity sets. We start by assuming a true model and generate simulated datasets from it. Each dataset is then used to construct an ambiguity set and a safe estimate of policy return. The performance of the methods is measured using the average of the absolute errors of the estimates compared with the true returns of the \emph{optimal} policies. All of our experiments use a 95\% confidence for the safety of the estimates.
	
	We compare ambiguity sets constructed using BCI, RSVF, with the Hoeffding sets. To reduce the conservativeness of Hoeffding sets when transition probabilities are sparse, we use a modification inspired by the Good-Turing bounds~\citep{Taleghan2015}. The modification is to assume that any transitions from $s,a$ to $s'$ are impossible if they are missing in the dataset $\dataset$. We also compare with the ``Hoeffding Monotone'' formulation $\asetht$ even when there is no guarantee that the value function is really monotone. This helps us to quantify the limitations of using concentration inequalities. Finally, we compare the results with the ``Mean Transition'' which solves the expected model $\bar{p}_{s,a}$ and provides no safety guarantees.
	
	\begin{figure}
		\centering
		\includegraphics[width=0.7\linewidth]{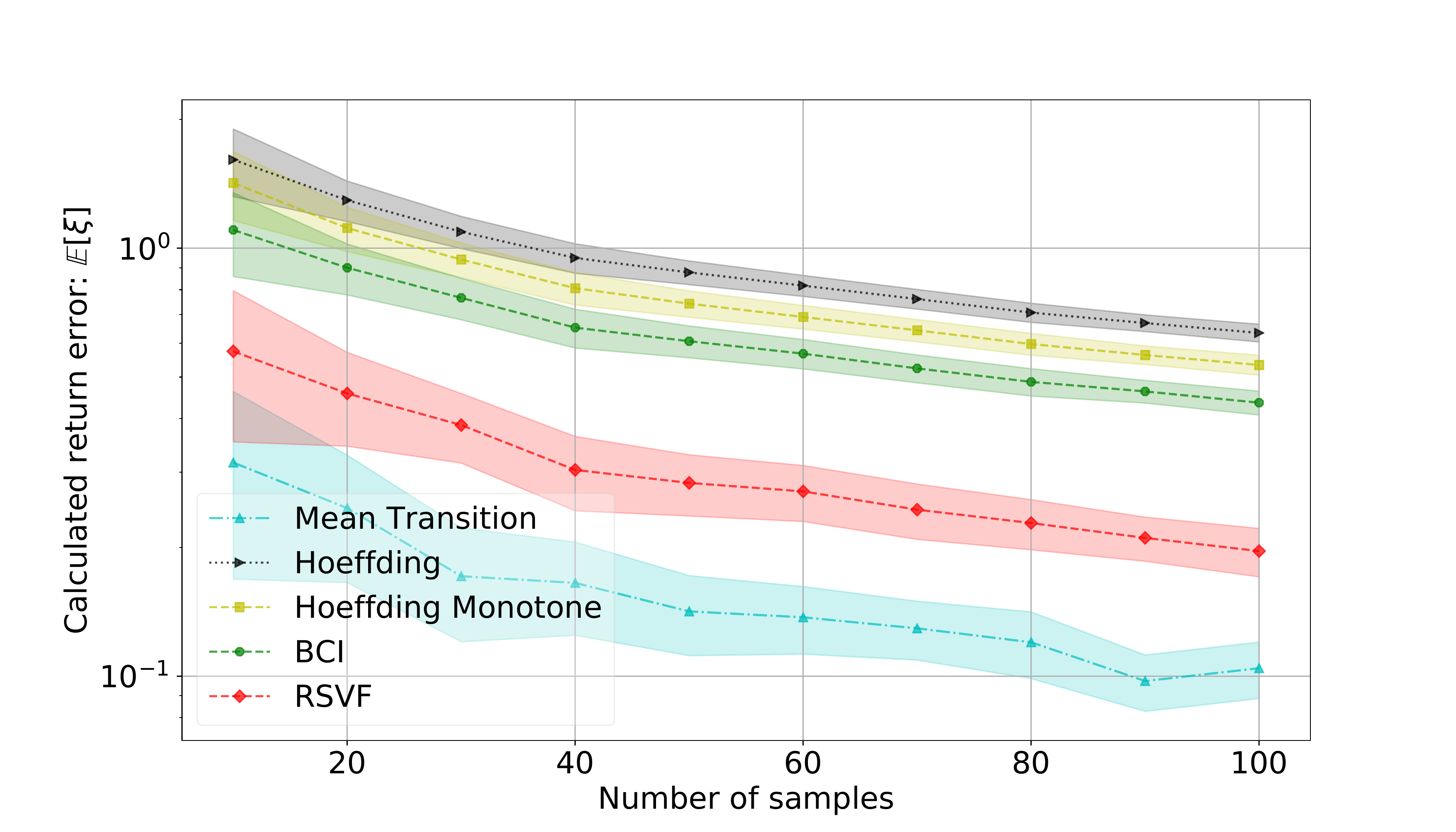}
		\caption{Expected regret of safe estimates with $95\%$ confidence regions for the Bellman update with an uninformative prior.}    
		\label{fig:single_state_dirichlet_errors}
	\end{figure}
	
	
	
	We do not evaluate the computational complexity of the methods since they target problems constrained by data and not computation. The Bayesian methods are generally more computationally demanding but the scale depends significantly on the type of the prior model used. All Bayesian methods draw $1,000$ samples from the posterior for each state and action.
	
	\subsection{Bellman Update} \label{subsec:single_state}
	
	In this section, we consider a transition from a single state $s_0$ and action $a_0$ to 5 states $s_1, \ldots, s_5$. The value function for the states $s_1, \ldots, s_5$ is fixed to be $[1,2,3,4,5]$. RSVF is run for a single iteration with the given value function. The single iteration of RSVF in this simplistic setting helps to quantify the possible benefit of using RSVF-style methods over BCI. The ground truth is generated from the corresponding prior for each one of the problems.
	
	\paragraph{Uninformative Dirichlet Priors}

	\begin{figure}
		\centering
		\includegraphics[width=0.6\linewidth]{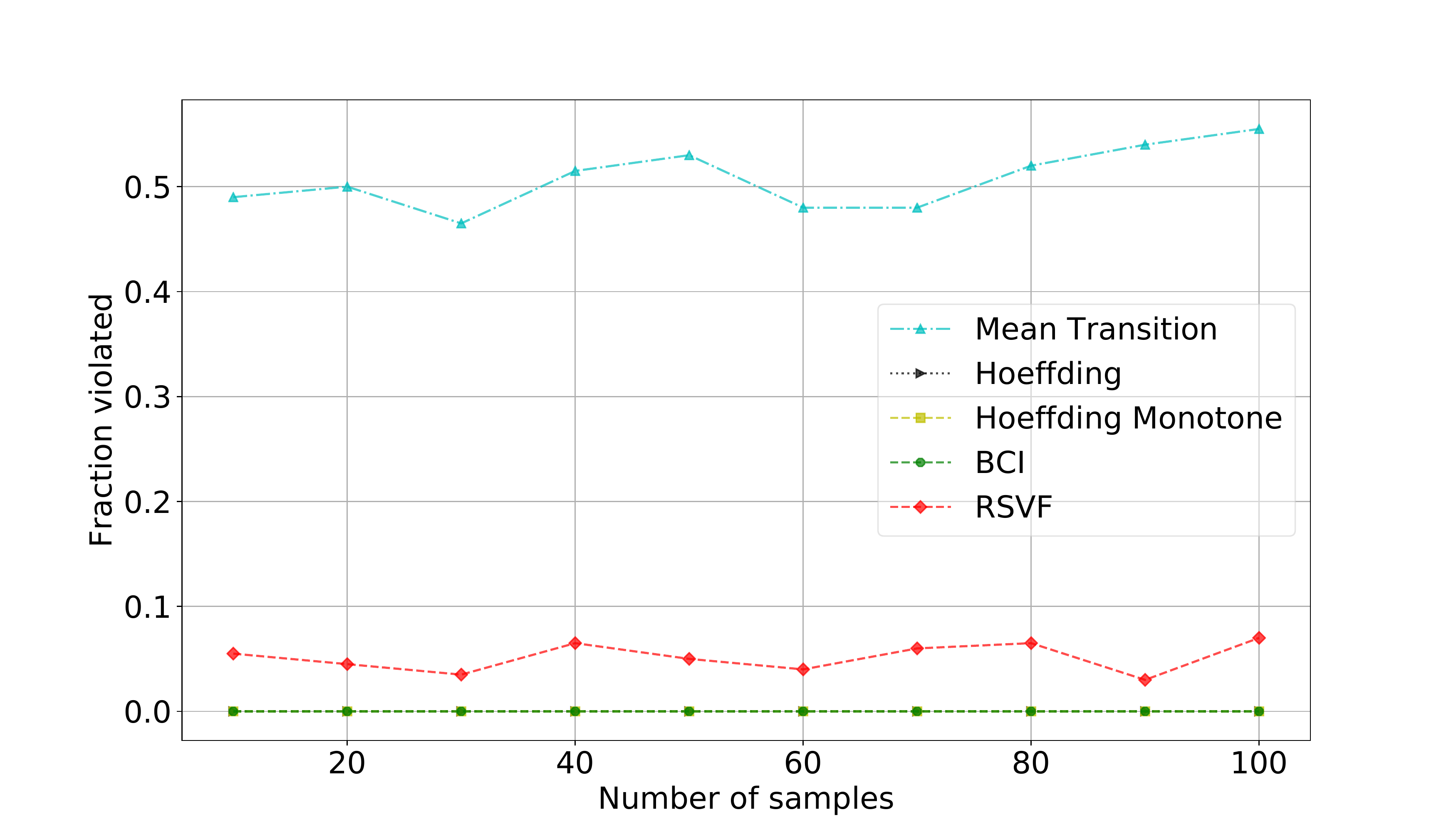}
		\caption{Rate of violations of the safety requirement for the single-state estimation with a uniform Dirichlet prior.}
		\label{fig:single_state_dirichlet_violations}
	\end{figure} 
	
		\begin{figure}
		\centering
		\includegraphics[width=0.6\linewidth]{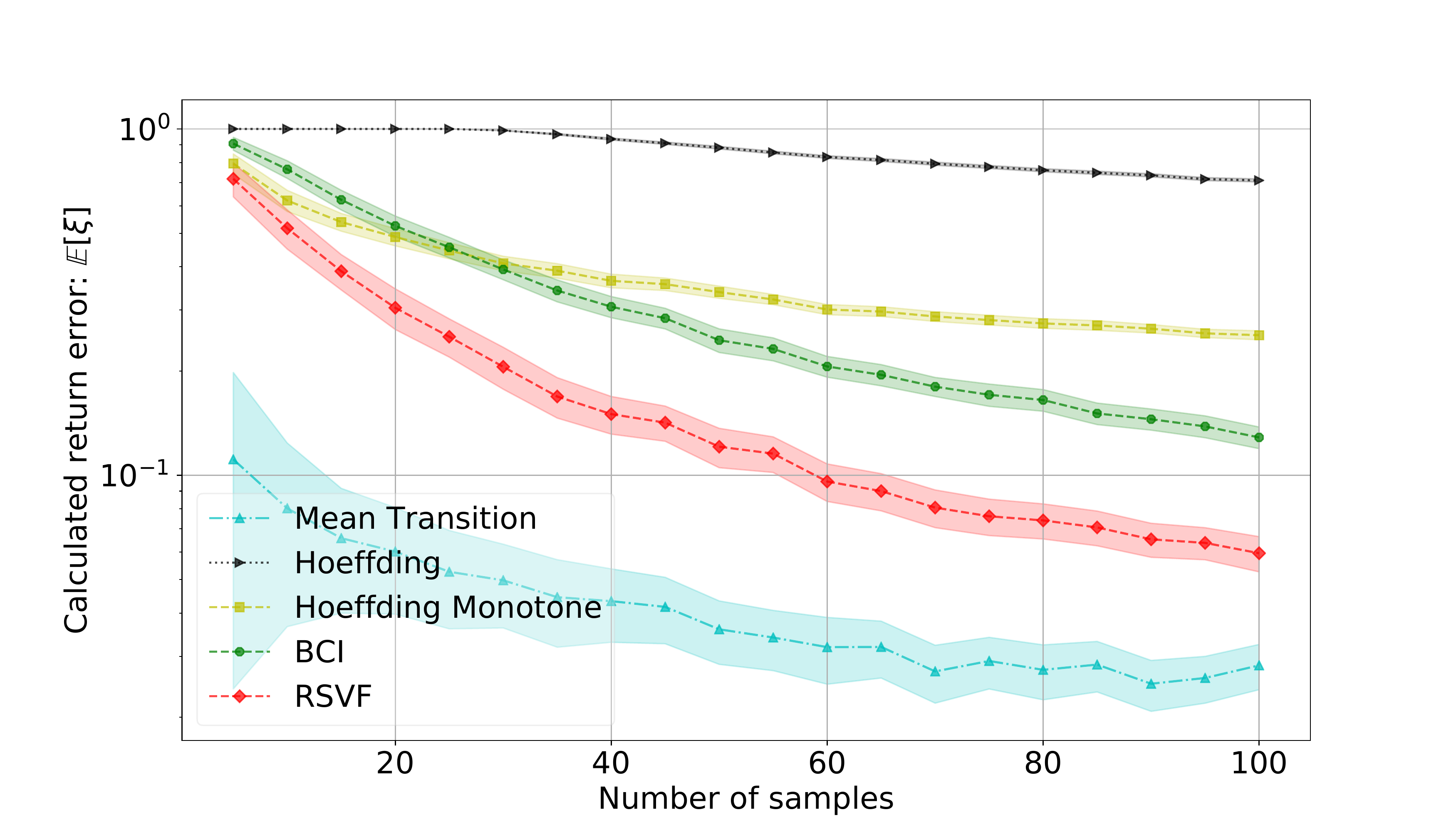}
		\caption{Expected regret of safe estimates with $95\%$ confidence regions for the Bellman update with an informative prior.}
		\label{fig:single_state_normal_errors}
	\end{figure}

	This setting considers a uniform Dirichlet distribution with $\alpha = [1,1,1,1,1]$ as the prior. This prior provides little information. \Cref{fig:single_state_dirichlet_errors} compares the computed robust return errors. The value $\xi$ represents the regret of predicted returns, which is the absolute difference between the \emph{true} optimal value and the robust estimate: $\xi = \abs{\rho(\pi\opt_{P\opt}, P\opt) - \tilde\rho(\hat\pi\opt)}$. Here, $\tilde{\rho}$ is the robust estimate and $\hat\pi\opt$ is the optimal robust solution. The smaller the value, the tighter and less conservative the safe estimate is. \Cref{fig:single_state_dirichlet_violations} shows the rate of safety violations: $\P_{\dataset}[\tilde\rho(\hat{\pi}\opt) > \rho(\hat{\pi}\opt, P\opt) ~|~ P\opt]$. The number of samples is the size of dataset $\dataset$. All results are computed by averaging over 200 simulated datasets of the given size generated from the ground-truth $P\opt$.

	The results show that BCI improves on both types Hoeffding bounds and RSVF further improves on BCI. The mean estimate provides the tightest bounds, but \cref{fig:single_state_dirichlet_violations} demonstrates that it does not provide any meaningful safety guarantees. \Cref{fig:single_state_dirichlet_violations} also provides insights into how RSVF improves on the other methods. Because the goal is to guarantee estimates are computed with 95\% confidence, one would expect the safety guarantees to be violated about $5\%$ of the time. BCI and Hoeffding solutions violate the safety requirements $0\%$ of the time. RSVF is optimal in this setting and achieves the desired $5\%$ violation.
	
	\begin{figure}
		\centering
		\includegraphics[width=0.6\linewidth]{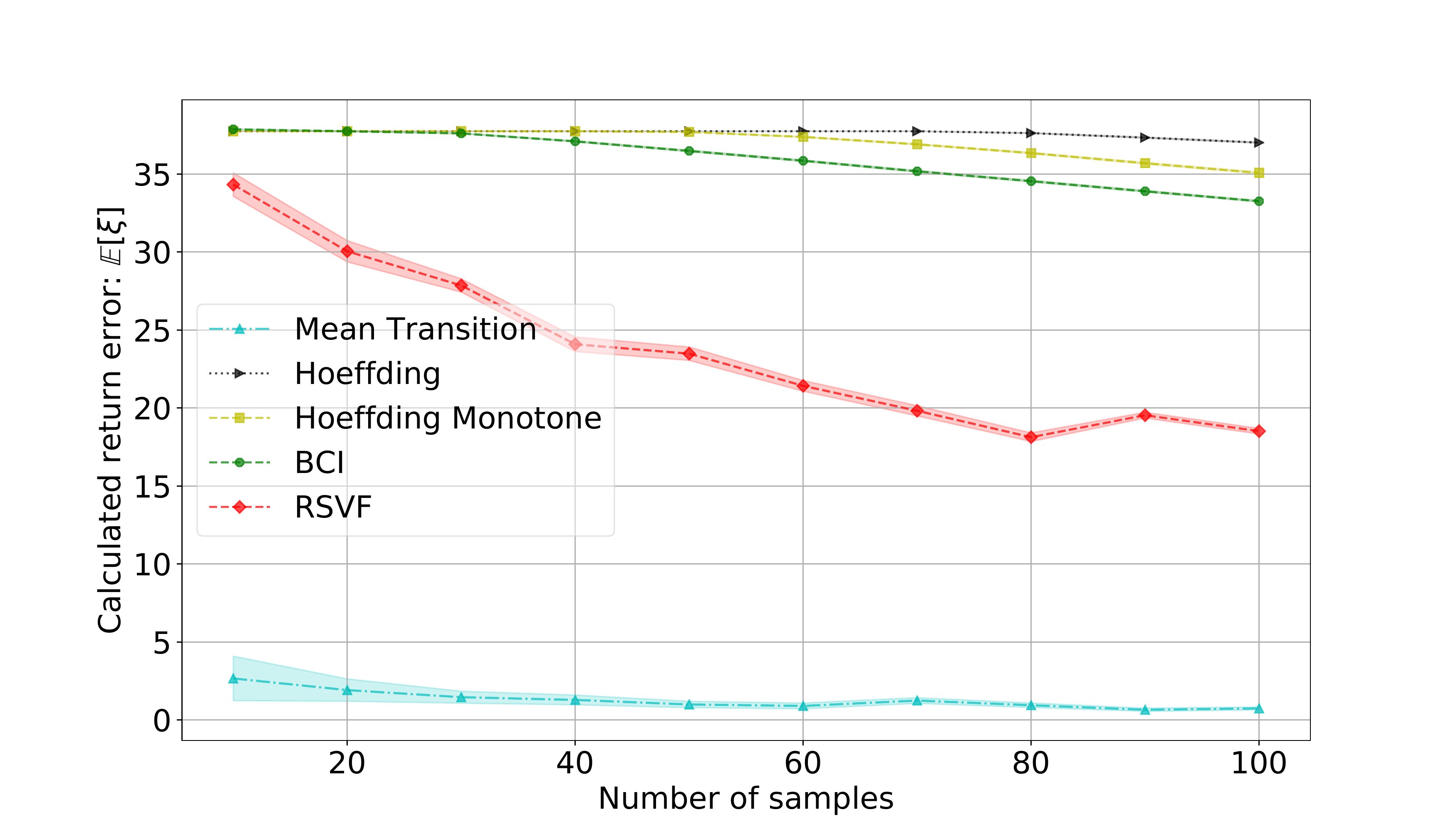}
		\caption{Expected regret of safe estimates with $95\%$ confidence regions for the RiverSwim: an MDP with an uninformative prior.}
		\label{fig:swimmer_approxim}
	\end{figure}
	
	\paragraph{Informative Gaussian Priors}     
	To evaluate the effect of using an informative prior, we use a problem inspired by inventory optimization. The states $s_1,\ldots,s_5$ represent inventory levels. The inventory level corresponds to the state index ($1$ in the state $s_1$) except that the inventory in the current state $s_0$ is $5$. The demand is assumed to be Normally distributed with an unknown mean $\mu$ and a \emph{known} standard deviation $\sigma = 1$. The prior over $\mu$ is Normal with the mean $\mu_0 = 3$ and, therefore, the posterior over $\mu$ is also Normal. The current action assumes that no product is ordered and, therefore, only the demand is subtracted from $s_0$. 
	
	\Cref{fig:single_state_normal_errors} compares the regret of safe estimates which were generated identically to the uninformative example. It shows that with an informative prior, BCI performs significantly better than Hoeffding bounds. RSVF provides still tighter bounds than BCI. The violations plot (not shown) is almost identical to \cref{fig:single_state_dirichlet_violations}.
	
	\subsection{Full MDP} \label{subsec:full_mdp}
	
	In this section, we evaluate the methods using MDPs with relatively small state-spaces. They can be used with certain types of value function approximation, like aggregation~\citep{Petrik2014}, but we evaluate them only on tabular problems to prevent approximation errors from skewing the results. To prevent the sampling policy from influencing the results, each dataset $\dataset$ has the same number of samples from each state.
	
	\paragraph{Uninformative Prior} We first use the standard RiverSwim domain for the evaluation~\citep{Strehl2008}. The methods are evaluated identically to the Bellman update above. That is, we generate synthetic datasets from the ground truth and then compare expected regret of the robust estimate with respect to the true return of the \emph{optimal} policy for the ground truth. As the prior, we use the uniform Dirichlet distribution over all states. \Cref{fig:swimmer_approxim} shows the expected robust regret over $100$ repetitions. The x-axis represents the number of samples in $\dataset$ for each state. It is apparent that BCI improves only slightly on the Hoeffding sets since the prior is not informative. RSVF, on the other hand, shows a significant improvement over BCI. All robust methods have safety violations of $0\%$ indicating that even RSVF is unnecessarily conservative here.
	
	\begin{figure}
		\centering
		\includegraphics[width=0.6\linewidth]{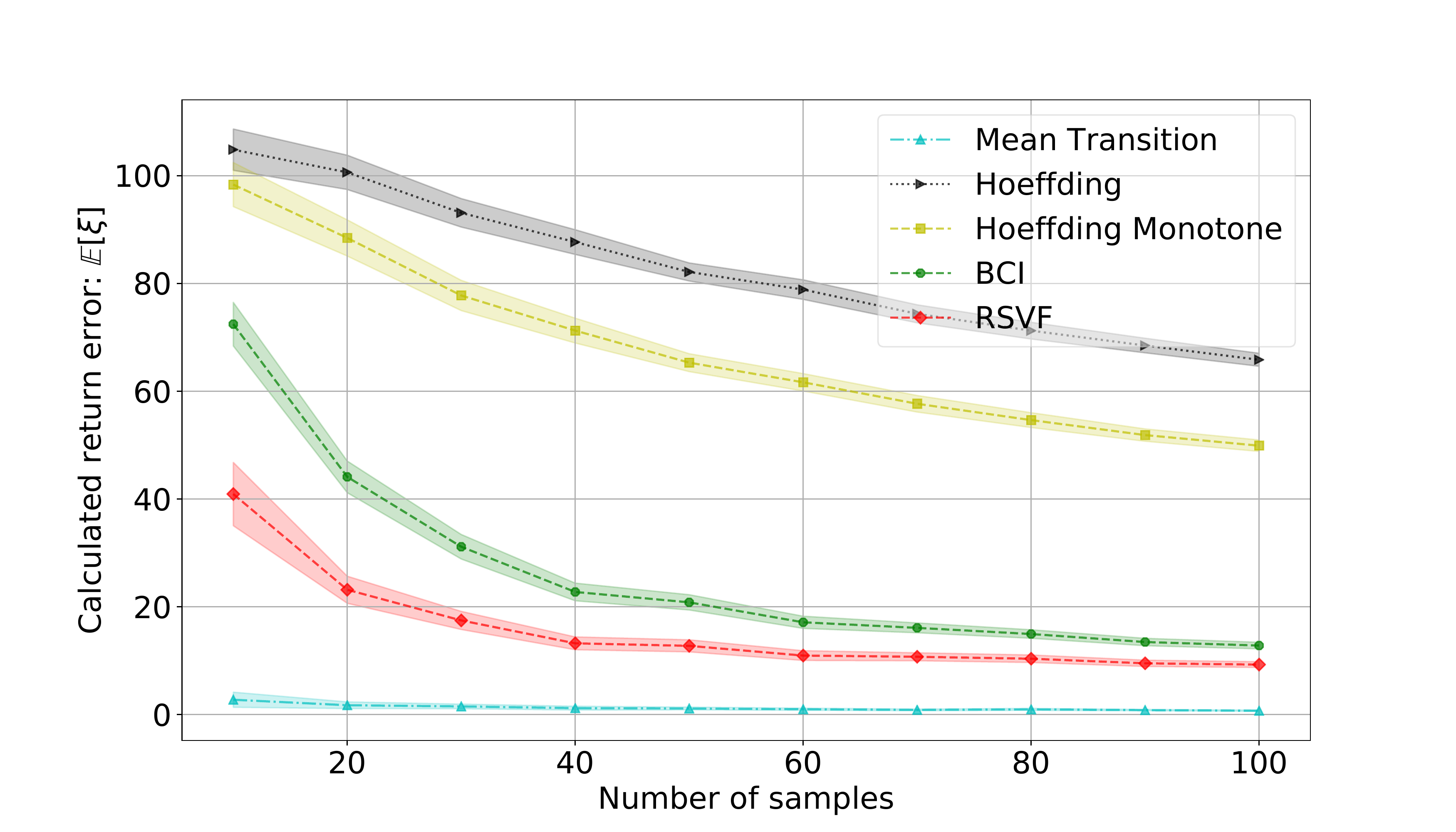}
		\caption{Expected regret of safe estimates with $90\%$ confidence regions for the ExpPopulation: an MDP with an informative prior.}
		\label{fig:population_approxim}
	\end{figure}
	\paragraph{Informative Prior} Next, we evaluate RSVF on the MDP model of a simple exponential population model~\citep{Tirinzoni2018}. Robustness plays an important role in ecological models because they are often complex, stochastic, and data collection is expensive. Yet, it is important that the decisions are robust due to their long term impacts.
	
	We only outline the population model here and refer the interested reader to \citeasnoun{Tirinzoni2018} for more details. The population $N_t$ of a species at time $t$ evolves according to the exponential dynamics $N_{t+1} = \min{(\lambda_t N_t, K)}$.  Here, $\lambda$ is the growth rate and $K$ is the carrying capacity of the environment. A manager must decide, at each time $t$, whether to apply a treatment that reduces the growth rate $\lambda$. The growth rate $\lambda_t$ is defined as:    $\lambda_t = \bar{\lambda} - z_t N_t\beta_1 - z_t\max{(0, N_t-\bar{N})^2}\beta_2 + \mathcal{N}(0,\sigma_y^2)$, where $\beta_1$ and $\beta_2$ are the coefficients of treatment effectiveness and $z_t$ is the indicator of treatment. A noisy estimate $y_t$ of the population $N_t$ is observed: $y_t \sim N_t + \mathcal{N}(0,\sigma_y^2)$. The state in the MDP is the population $y_t$ discretized to $20$ values. There are two actions whether to apply the treatment. The rewards capture the costs of high population and the treatment application. The exponential growth model is used as the prior and all priors and posteriors are Normally distributed.
	
	
	\Cref{fig:population_approxim} shows the average regret of the safe predictions. BCI can leverage the prior information to compute tighter bounds, but RSVF further improves on BCI. The rate of safety violations is again $0\%$ for all robust methods. 
	
	
	\section{Summary and Conclusion}
	
	This paper proposes new Bayesian algorithms for constructing ambiguity sets in RMDPs, improving over standard distribution-free methods. BCI makes it possible to flexibly incorporate prior domains knowledge and is easy to generalize to other shapes of ambiguity sets (like $L_2$) without having to prove new concentration inequalities. Finally, RSVF improves on BCI by constructing tighter ambiguity sets that are not confidence regions. Our experimental results and theoretical analysis indicate that the new ambiguity sets provide much tighter safe return estimates. The only drawbacks of the Bayesian methods are that they need priors and may increase the computational complexity.
	
	RSVF can be improved in several ways. Our experimental results show that the method is still too conservative since it has no safety violations. Generalizing beyond $L_1$ sets and rectangularity are likely to mitigate the conservativeness. The distribution-free ambiguity sets can probably be tightened by using the Bernstein inequality.
	
	\section*{Acknowledgments}
	
	We thank Vishal Gupta for enlightening comments on the topic of Bayesian ambiguity sets for distributionally robust optimization.
	
	\bibliographystyle{icml}
	\bibliography{marek}

\begin{thebibliography}{43}
\providecommand{\natexlab}[1]{#1}
\providecommand{\url}[1]{\texttt{#1}}
\expandafter\ifx\csname urlstyle\endcsname\relax
  \providecommand{\doi}[1]{doi: #1}\else
  \providecommand{\doi}{doi: \begingroup \urlstyle{rm}\Url}\fi

\bibitem[Auer et~al.(2010)Auer, Jaksch, and Ortner]{Auer2010a}
Auer, P., Jaksch, T., and Ortner, R.
\newblock {Near-optimal regret bounds for reinforcement learning}.
\newblock \emph{Journal of Machine Learning Research}, 11\penalty0
  (1):\penalty0 1563--1600, 2010.

\bibitem[Bagnell et~al.(2001)Bagnell, Ng, and Schneider]{Bagnell2001b}
Bagnell, J.~A., Ng, A.~Y., and Schneider, J.~G.
\newblock {Solving Uncertain Markov Decision Processes}.
\newblock \emph{Carnegie Mellon Research Showcase}, pp.\  948--957, 2001.

\bibitem[Ben-Tal et~al.(2009)Ben-Tal, {El Ghaoui}, and Nemirovski]{Ben-Tal2009}
Ben-Tal, A., {El Ghaoui}, L., and Nemirovski, A.
\newblock \emph{{Robust Optimization}}.
\newblock Princeton University Press, 2009.

\bibitem[Bertsekas \& Tsitsiklis(1996)Bertsekas and Tsitsiklis]{Bertsekas1996}
Bertsekas, D.~P. and Tsitsiklis, J.~N.
\newblock \emph{{Neuro-dynamic programming}}.
\newblock 1996.

\bibitem[Bertsimas et~al.(2017)Bertsimas, Kallus, and Gupta]{Bertsimas2017}
Bertsimas, D., Kallus, N., and Gupta, V.
\newblock \emph{{Data-driven robust optimization}}.
\newblock Springer Berlin Heidelberg, 2017.

\bibitem[Delgado et~al.(2016)Delgado, {De Barros}, Dias, and
  Sanner]{Delgado2016}
Delgado, K.~V., {De Barros}, L.~N., Dias, D.~B., and Sanner, S.
\newblock {Real-time dynamic programming for Markov decision processes with
  imprecise probabilities}.
\newblock \emph{Artificial Intelligence}, 230:\penalty0 192--223, 2016.

\bibitem[Dietterich et~al.(2013)Dietterich, Taleghan, and
  Crowley]{Dietterich2013}
Dietterich, T., Taleghan, M., and Crowley, M.
\newblock {PAC optimal planning for invasive species management: Improved
  exploration for reinforcement learning from simulator-defined MDPs.}
\newblock \emph{AAAI}, 2013.

\bibitem[Gelman et~al.(2014)Gelman, Carlin, Stern, and Rubin]{Gelman2014}
Gelman, A., Carlin, J.~B., Stern, H.~S., and Rubin, D.~B.
\newblock \emph{{Bayesian Data Analysis}}.
\newblock Chapman and Hall/CRC, 3rd edition, 2014.

\bibitem[Goyal \& Grand-Clement(2018)Goyal and Grand-Clement]{Goyal2018}
Goyal, V. and Grand-Clement, J.
\newblock {Robust Markov Decision Process: Beyond Rectangularity}.
\newblock Technical report, 2018.

\bibitem[Gupta(2015)]{Gupta2015}
Gupta, V.
\newblock {Near-Optimal Bayesian Ambiguity Sets for Distributionally Robust
  Optimization}.
\newblock 2015.

\bibitem[Hanasusanto \& Kuhn(2013)Hanasusanto and Kuhn]{Hanasusanto2013}
Hanasusanto, G. and Kuhn, D.
\newblock {Robust Data-Driven Dynamic Programming}.
\newblock In \emph{Advances in Neural Information Processing Systems (NIPS)},
  2013.

\bibitem[Ho et~al.(2018)Ho, Petrik, and Wiesemann]{Ho2018}
Ho, C.~P., Petrik, M., and Wiesemann, W.
\newblock {Fast Bellman Updates for Robust MDPs}.
\newblock In \emph{International Conference on Machine Learning (ICML)},
  volume~80, pp.\  1979--1988, 2018.

\bibitem[Iyengar(2005)]{Iyengar2005}
Iyengar, G.~N.
\newblock {Robust dynamic programming}.
\newblock \emph{Mathematics of Operations Research}, 30\penalty0 (2):\penalty0
  257--280, 2005.

\bibitem[Jaksch et~al.(2010)Jaksch, Ortner, and Auer]{Auer2010}
Jaksch, T., Ortner, R., and Auer, P.
\newblock {Near-optimal Regret Bounds for Reinforcement Learning}.
\newblock \emph{Journal of Machine Learning Research}, 11\penalty0
  (1):\penalty0 1563--1600, 2010.

\bibitem[Jiang \& Li(2015)Jiang and Li]{Jiang2015b}
Jiang, N. and Li, L.
\newblock {Doubly Robust Off-policy Value Evaluation for Reinforcement
  Learning}.
\newblock In \emph{International Conference on Machine Learning (ICML)}, 2015.

\bibitem[Kalyanasundaram et~al.(2002)Kalyanasundaram, Chong, and
  Shroff]{Kalyanasundaram2002}
Kalyanasundaram, S., Chong, E. K.~P., and Shroff, N.~B.
\newblock {Markov decision processes with uncertain transition rates:
  Sensitivity and robust control}.
\newblock In \emph{IEEE Conference on Decision and Control}, pp.\  3799--3804,
  2002.

\bibitem[Lange et~al.(2012)Lange, Gabel, and Riedmiller]{Lange2012}
Lange, S., Gabel, T., and Riedmiller, M.
\newblock {Batch Reinforcement Learning}.
\newblock In \emph{Reinforcement Learning}, pp.\  45--73. 2012.

\bibitem[Laroche \& Trichelair(2018)Laroche and Trichelair]{Laroche2017}
Laroche, R. and Trichelair, P.
\newblock {Safe Policy Improvement with Baseline Bootstrapping}, 2018.

\bibitem[{Le Tallec}(2007)]{LeTallec2007}
{Le Tallec}, Y.
\newblock \emph{{Robust, Risk-Sensitive, and Data-driven Control of Markov
  Decision Processes}}.
\newblock PhD thesis, MIT, 2007.

\bibitem[Li et~al.(2015)Li, Munos, and Szepesv{\'{a}}ri]{Li2015}
Li, L., Munos, R., and Szepesv{\'{a}}ri, C.
\newblock {Toward Minimax Off-policy Value Estimation}.
\newblock In \emph{International Conference on Artificial Intelligence and
  Statistics (AISTATS)}, 2015.

\bibitem[Lim et~al.(2013)Lim, Xu, and Mannor]{Lim2013}
Lim, S.~H., Xu, H., and Mannor, S.
\newblock {Reinforcement Learning in Robust Markov Decision Processes}.
\newblock In \emph{Advances in Neural Information Processing Systems (NIPS)},
  2013.

\bibitem[Mannor et~al.(2012)Mannor, Mebel, and Xu]{Mannor2012}
Mannor, S., Mebel, O., and Xu, H.
\newblock {Lightning does not strike twice: Robust MDPs with coupled
  uncertainty}.
\newblock In \emph{International Conference on Machine Learning (ICML)}, 2012.

\bibitem[Mannor et~al.(2016)Mannor, Mebel, and Xu]{Mannor2016}
Mannor, S., Mebel, O., and Xu, H.
\newblock {Robust MDPs with k-rectangular uncertainty}.
\newblock \emph{Mathematics of Operations Research}, 41\penalty0 (4):\penalty0
  1484--1509, 2016.

\bibitem[Munos et~al.(2016)Munos, Stepleton, Harutyunyan, and
  Bellemare]{Munos2016}
Munos, R., Stepleton, T., Harutyunyan, A., and Bellemare, M.~G.
\newblock {Safe and Efficient Off-Policy Reinforcement Learning}.
\newblock In \emph{Conference on Neural Information Processing Systems (NIPS)},
  2016.

\bibitem[Murphy(2012)]{Murphy2012}
Murphy, K.
\newblock \emph{{Machine Learning: A Probabilistic Perspective}}.
\newblock 2012.

\bibitem[Nilim \& {El Ghaoui}(2005)Nilim and {El Ghaoui}]{Nilim2005}
Nilim, A. and {El Ghaoui}, L.
\newblock {Robust control of Markov decision processes with uncertain
  transition matrices}.
\newblock \emph{Operations Research}, 53\penalty0 (5):\penalty0 780--798, 2005.

\bibitem[Petrik(2012)]{Petrik2012}
Petrik, M.
\newblock {Approximate dynamic programming by minimizing distributionally
  robust bounds}.
\newblock In \emph{International Conference of Machine Learning (ICML)}, 2012.

\bibitem[Petrik \& Subramanian(2014)Petrik and Subramanian]{Petrik2014}
Petrik, M. and Subramanian, D.
\newblock {RAAM : The benefits of robustness in approximating aggregated MDPs
  in reinforcement learning}.
\newblock In \emph{Neural Information Processing Systems (NIPS)}, 2014.

\bibitem[Petrik et~al.(2016)Petrik, {Mohammad Ghavamzadeh}, and
  Chow]{Petrik2016a}
Petrik, M., {Mohammad Ghavamzadeh}, and Chow, Y.
\newblock {Safe Policy Improvement by Minimizing Robust Baseline Regret}.
\newblock In \emph{Advances in Neural Information Processing Systems (NIPS)},
  2016.

\bibitem[Puterman(2005)]{Puterman2005}
Puterman, M.~L.
\newblock \emph{{Markov decision processes: Discrete stochastic dynamic
  programming}}.
\newblock 2005.

\bibitem[Shapiro et~al.(2014)Shapiro, Dentcheva, and Ruszczynski]{Shapiro2014}
Shapiro, A., Dentcheva, D., and Ruszczynski, A.
\newblock \emph{{Lectures on stochastic programming: Modeling and theory}}.
\newblock 2014.

\bibitem[Strehl \& Littman(2008)Strehl and Littman]{Strehl2008}
Strehl, A. and Littman, M.
\newblock {An analysis of model-based Interval Estimation for Markov Decision
  Processes}.
\newblock \emph{Journal of Computer and System Sciences}, 74:\penalty0
  1309--1331, 2008.

\bibitem[Strehl(2007)]{Strehl2008a}
Strehl, A.~L.
\newblock \emph{{Probably Approximately Correct (PAC) Exploration in
  Reinforcement Learning}}.
\newblock PhD thesis, Rutgers University, 2007.

\bibitem[Sutton \& Barto(1998)Sutton and Barto]{Sutton1998}
Sutton, R.~S. and Barto, A.
\newblock \emph{{Reinforcement learning}}.
\newblock 1998.

\bibitem[Taleghan et~al.(2015)Taleghan, Dietterich, Crowley, Hall, and
  Albers]{Taleghan2015}
Taleghan, M.~A., Dietterich, T.~G., Crowley, M., Hall, K., and Albers, H.~J.
\newblock {PAC Optimal MDP Planning with Application to Invasive Species
  Management}.
\newblock \emph{Journal of Machine Learning Research}, 16:\penalty0 3877--3903,
  2015.

\bibitem[Tamar et~al.(2014)Tamar, Mannor, and Xu]{Tamar2014a}
Tamar, A., Mannor, S., and Xu, H.
\newblock {Scaling up Robust MDPs Using Function Approximation}.
\newblock In \emph{International Conference of Machine Learning (ICML)}, 2014.

\bibitem[Thomas \& Brunskill(2016)Thomas and Brunskill]{Thomas2016}
Thomas, P.~S. and Brunskill, E.
\newblock {Data-efficient off-policy policy evaluation for reinforcement
  learning}.
\newblock In \emph{International Conference of Machine Learning (ICML)}, 2016.

\bibitem[Thomas et~al.(2015)Thomas, Teocharous, and Ghavamzadeh]{Thomas2015}
Thomas, P.~S., Teocharous, G., and Ghavamzadeh, M.
\newblock {High Confidence Off-Policy Evaluation}.
\newblock In \emph{Annual Conference of the AAAI}, 2015.

\bibitem[Tirinzoni et~al.(2018)Tirinzoni, Milano, Chen, and
  Ziebart]{Tirinzoni2018}
Tirinzoni, A., Milano, P., Chen, X., and Ziebart, B.~D.
\newblock {Policy-Conditioned Uncertainty Sets for Robust Markov Decision
  Processes}.
\newblock In \emph{Neural Information Processing Systems (NIPS)}, 2018.

\bibitem[Weissman et~al.(2003)Weissman, Ordentlich, Seroussi, Verdu, and
  Weinberger]{Weissman2003xx}
Weissman, T., Ordentlich, E., Seroussi, G., Verdu, S., and Weinberger, M.~J.
\newblock {Inequalities for the L1 deviation of the empirical distribution}.
\newblock 2003.

\bibitem[Wiesemann et~al.(2013)Wiesemann, Kuhn, and Rustem]{Wiesemann2013}
Wiesemann, W., Kuhn, D., and Rustem, B.
\newblock {Robust Markov decision processes}.
\newblock \emph{Mathematics of Operations Research}, 38\penalty0 (1):\penalty0
  153--183, 2013.

\bibitem[Xu \& Mannor(2006)Xu and Mannor]{Xu2006}
Xu, H. and Mannor, S.
\newblock {The robustness-performance tradeoff in Markov decision processes}.
\newblock \emph{Advances in Neural Information Processing Systems (NIPS)},
  2006.

\bibitem[Xu \& Mannor(2009)Xu and Mannor]{Xu2009}
Xu, H. and Mannor, S.
\newblock {Parametric regret in uncertain Markov decision processes}.
\newblock In \emph{IEEE Conference on Decision and Control (CDC)}, pp.\
  3606--3613, 2009.

\end{thebibliography}
	
	\newpage
	\appendix
	\onecolumn
	
	\section{Technical Results}
	
	The following proposition shows that the guarantee of a safe estimate on the return is achieved when the true transition model is contained in the ambiguity set.
	\begin{lemma} \label{prop:confidence_interval_freq}
		Suppose that an ambiguity set $\aset$ satisfies $\P_{\dataset}\left[ p_{s,a}\opt \in \aset_{s,a} ~\vert~ P\opt \right] \ge 1-\delta/(SA)$ for each state $s$ and action $a$. Then:
		\[
		\P_{\dataset} \left[\hat{v}^\pi_\aset \le v_{P\opt}^{\pi}, \; \forall \pi\in\Pi ~|~ P\opt \right] \geq 1-\delta~.
		\]
	\end{lemma}
	\begin{proof}
		We omit $\aset$ and $P\opt$ from the notation in the proof since they are fixed.
		From \cref{prop:single_to_many}, we have that $\hat{v}^\pi \le v^{\pi}$ if 
		\[ \RBU^\pi \hat{v}^\pi \le \BU^\pi \hat{v}^\pi~. \]
		That is, for each state $s$ and action $a$:
		\[ \min_{p \in \aset_{s,a}} p\tr \hat{v}^\pi \le (p\opt_{s,a})\tr \hat{v}^\pi.  \]
		Using the identity above, the probability that the robust value function is a lower bound can be bounded as follows:
		\begin{gather*}
		\P_{\dataset} \left[\hat{v}^\pi_\aset \le v_P^{\pi}, \; \forall \pi\in\Pi ~|~ P\opt \right] = \P_{\dataset} \left[\min_{p \in \aset_{s,a}} p\tr \hat{v}^\pi \le (p\opt_{s,a})\tr  \hat{v}^\pi, \; \forall \pi\in\Pi, s\in\states,a\in\actions ~|~ P\opt \right] \ge \\
		\ge \P_{\dataset} \left[(p_{s,a}\opt)\tr \hat{v}^\pi \le (p\opt_{s,a})\tr  \hat{v}^\pi, \; \forall \pi\in\Pi, s\in\states,a\in\actions ~|~ P\opt\in \aset,  P\opt \right] \P_{\dataset} \left[P\opt\in \aset ~|~  P\opt  \right] + \\
		+ \P_{\dataset} \left[P\opt\notin \aset ~|~  P\opt  \right] \ge
		1\, \P_{\dataset} \left[P\opt\in \aset ~|~  P\opt  \right] + 0\, \P_{\dataset} \left[P\opt\notin \aset ~|~  P\opt  \right] \ge\\
		\ge \P_{\dataset} \left[P\opt\in \aset ~|~  P\opt  \right]~.
		\end{gather*}
		Now, from the union bound over all states and actions, we get:
		\[ \P_{\dataset} \left[  \hat{v}^\pi > v^{\pi} | P\opt \right] \le \P_{\dataset} \left[P\opt\notin \aset ~|~  P\opt  \right] \le \sum_{s\in\states} \sum_{a\in\actions} \P_{\dataset} \left[ p\opt_{s,a} \notin \aset_{s,a} ~|~ P\opt \right] \le \delta~, \]
		which completes the proof.
	\end{proof}
	
	The next proposition is the Bayesian equivalent of \cref{prop:confidence_interval_freq}.
	\begin{lemma} \label{prop:confidence_interval}
		Suppose that an ambiguity set $\aset$ satisfies $\P_{P\opt}\left[ p_{s,a}\opt \in \aset_{s,a} ~\vert~ \dataset \right] \ge 1-\delta/(SA)$ for each state $s$ and action $a$. Then:
		\[
		\P_{P\opt} \left[\hat{v}^\pi_\aset \le v_{P\opt}^{\pi}, \; \forall \pi\in\Pi ~|~ \dataset \right] \geq 1-\delta~.
		\]
	\end{lemma}
	\begin{proof}
		We omit $\aset$ and $P\opt$ from the notation in the proof since they are fixed.
		From \cref{prop:single_to_many}, we have that $\hat{v}^\pi \le v^{\pi}$ if 
		\[ \RBU^\pi \hat{v}^\pi \le \BU^\pi \hat{v}^\pi~. \]
		That is, for each state $s$ and action $a$:
		\[ \min_{p \in \aset_{s,a}} p\tr \hat{v}^\pi \le (p\opt_{s,a})\tr \hat{v}^\pi.  \]
		Using the identity above, the probability that the robust value function is a lower bound can be bounded as follows:
		\begin{gather*}
		\P_{P\opt} \left[\hat{v}^\pi_\aset \le v_P^{\pi}, \; \forall \pi\in\Pi ~|~ \dataset \right] = \P_{P\opt} \left[\min_{p \in \aset_{s,a}} p\tr \hat{v}^\pi \le (p\opt_{s,a})\tr  \hat{v}^\pi, \; \forall \pi\in\Pi, s\in\states,a\in\actions ~|~ \dataset \right] \ge \\
		\ge \P_{P\opt} \left[(p_{s,a}\opt)\tr \hat{v}^\pi \le (p\opt_{s,a})\tr  \hat{v}^\pi, \; \forall \pi\in\Pi, s\in\states,a\in\actions ~|~ P\opt\in \aset,  \dataset \right] \P_{P\opt} \left[P\opt\in \aset ~|~  \dataset  \right] + \\
		+ \P_{P\opt} \left[P\opt\notin \aset ~|~  \dataset  \right] \ge
		1\, \P_{P\opt} \left[P\opt\in \aset ~|~  \dataset  \right] + 0\, \P_{P\opt} \left[P\opt\notin \aset ~|~  \dataset  \right] \ge\\
		\ge \P_{P\opt} \left[P\opt\in \aset ~|~  \dataset  \right]~.
		\end{gather*}
		Now, from the union bound over all states and actions, we get:
		\[ \P_{P\opt} \left[  \hat{v}^\pi > v^{\pi} | \dataset \right] \le \P_{P\opt} \left[P\opt\notin \aset ~|~  \dataset  \right] \le \sum_{s\in\states} \sum_{a\in\actions} \P_{P\opt} \left[ p\opt_{s,a} \notin \aset_{s,a} ~|~ \dataset \right] \le \delta~, \]
		which completes the proof.
	\end{proof}

	\section{Technical Proofs} \label{app:proofs}
	
	\subsection{Proof of \cref{prop:single_to_many}}
	\begin{proof}
		Using the assumption $\RBU^\pi \hat{v}^\pi \le \BU^\pi \hat{v}^\pi$, and from $\hat{v}^\pi = \RBU^\pi \hat{v}^\pi$ and $v^\pi = \BU^\pi v^\pi$, we get by algebraic manipulation:    
		\[
		\hat{v}^\pi - v^{\pi} = \RBU^\pi \hat{v}^\pi - \BU_P^{\pi} v^{\pi} \le \BU^{\pi} \hat{v}^\pi - \BU^{\pi} v^{\pi} = \gamma P_{\pi} (\hat{v}^\pi - v^{\pi})~.
		\]
		Here, $P_\pi$ is the transition probability matrix for the policy $\pi$. Subtracting $\gamma P_{\pi}(\hat{v}^\pi - v^{\pi})$ from the above inequality gives:
		\[ (\eye - \gamma P_{\pi}) (\hat{v}^\pi - v^{\pi}) \le \zero ~,\]
		where $\eye$ is the identity matrix. Because the matrix $(\eye - \gamma P_{\pi\opt})^{-1}$ is monotone, as can be seen from its Neumann series, we get:
		\[ \hat{v}^\pi - v^{\pi} \le (\eye - \gamma P_{\pi})^{-1} \zero = \zero ~,\] 
		which proves the result.
	\end{proof}
	
	\subsection{Proof of \cref{cor:hoeffding_bound}}
	\begin{proof}
		The first part of the statement follows directly from \cref{prop:confidence_interval_freq} and \cref{lem:single_one}. The second part of the statement follows from the fact that the lower bound property holds uniformly across all policies.
	\end{proof}
	
	\subsection{Proof of \cref{cor:bci_bound}}
	\begin{proof}
		The first part of the statement follows directly from \cref{prop:confidence_interval} and the definition of $\psi^B_{s,a}$. The second part of the statement follows from the fact that the lower bound property holds uniformly across all policies.
	\end{proof}
	
	\section{$L_1$ Concentration Inequality Bounds}
	
	In this section, we describe a new elementary proof of a bound on the $L_1$ distance between the estimated transition probability distribution and the true one. It simplifies the proofs of \citeasnoun{Weissman2003xx} but also leads to coarser bounds. We include the proof here in order to derive the tighter bound in \cref{sec:improved_bounds}.  Note that in the frequentist setting the ambiguity set $\aset$ is a random variable that is a function of the dataset $\dataset$. 
	
	Recall that our ambiguity sets are defined as $L_1$ balls around the expected  transition probabilities $\bar{p}_{s,a}$:
	\begin{equation}  \label{eq:ambiguity_l1}
	\aset_{s,a} = \{p \in \Delta^\statecount \ss \norm{p - \bar{p}_{s,a} }_1 \le \psi_{s,a} \} ~. 
	\end{equation}
	\cref{prop:confidence_interval_freq} implies that the size of the $L_1$ balls must be chosen as follows:
	\begin{equation} \label{eq:rectangular_bound_union}
	\P\left[ \| \bar{p}(s,a) - p\opt(s,a)\|_1 \le \psi_{s,a} \, \right] \ge 1- \delta/(SA)~.
	\end{equation}
	
	We can now express the necessary size $\psi_{s,a}$ of the ambiguity sets in terms of $n_{s,a}$, which denotes the number of samples in $\dataset$ that originate with a state $s$ and an action $a$. 
	\begin{lemma}[$L_1$ Error bound] \label{lem:single_one}
		Suppose that $\bar{p}_{s,a}$ is the empirical estimate of the transition probability obtained from $n_{s,a}$ samples for each $s\in\states$ and $a\in\actions$. Then:
		\[ \P \left[ \| \bar{p}_{s,a} - p\opt_{s,a}\|_1 \ge \psi_{s,a} \, \right] \le (2^{S} - 2) \exp\left(-\frac{\psi_{s,a}^2 n_{s,a}}{2}\right) ~.\]
		Therefore, for any $\delta \in [0,1]$:
		\[ \P\left[ \| \bar{p}_{s,a} - p\opt_{s,a}\|_1 \le \sqrt{\frac{2}{n_{s,a}}  \log \frac{S A (2^{S} - 2)}{\delta} } \, \right] \le 1 - \delta / (S A)  ~.  \]
	\end{lemma}
	\begin{proof}
		To shorten the notation, we omit the indexes $s,a$ throughout the proof; for example $\bar{p}$ is used instead of the full $\bar{p}_{s,a}$. First, express the $L_1$ distance between two distributions $\bar{p}$ and $p\opt$ in terms of an optimization problem. Let $\one_{\mathcal{Q}} \in \Real^\states$ be the indicator vector for some subset $\mathcal{Q} \subset \states$. Then:
		\[
		\begin{aligned}
		\| \bar{p} - p\opt \|_1  &= \max_{z} \left\{ z\tr (\bar{p} - p\opt) \ss \| z \|_\infty \le 1 \right\} = \\
		&= \max_{\mathcal{Q} \in 2^\states} \left\{ \one_{\mathcal{Q}}\tr (\bar{p} - p\opt) - (\one-\one_{\mathcal{Q}})\tr(\bar{p} - p\opt)  \ss 0 < |\mathcal{Q}| < m \right\} \\
		&\stackrel{\text{(a)}}{=}  2 \max_{\mathcal{Q} \in 2^\states} \left\{ \one_{\mathcal{Q}}\tr (\bar{p} - p\opt) \ss 0 < |\mathcal{Q}| < m  \right\}  ~.
		\end{aligned}
		\]
		Here, (a) holds because $\one\tr (\bar{p} - p\opt) = 0$. Using the expression above, the target probability can be bounded as follows:
		\begin{align*}
		\P\left[ \Vert \bar{p} - p \opt \rVert_1 > \psi \right] &= \P \left[ 2 \max_{\mathcal{Q} \in 2^\states} \left\{\one_{\mathcal{Q}}\tr (\bar{p} - p\opt) \ss 0 < |\mathcal{Q}| < m \right\} > \psi \right] \\
		&\stackrel{\text{(a)}}{\le}  (|\mathcal{Q} |- 2)  \max_{ \mathcal{Q} \in 2^\states} \left\{ \P\left[ \one_{\mathcal{Q}}\tr (\bar{p} - p\opt) > \frac{\psi}{2} \right]  \ss 0 < |\mathcal{Q}| < m \right\} \\
		&\stackrel{\text{(b)}}{\le} (|\mathcal{Q}| - 2)  \exp\left(-\frac{\psi^2 n}{2} \right) = (2^{S} - 2)  \exp\left(-\frac{\psi^2 n}{2} \right)~.
		\end{align*}
		The inequality (a) follows from union bound and the inequality (b) follows from the Hoeffding's inequality since $\one_{\mathcal{Q}}\tr \bar{p} \in [0,1]$ for any $\mathcal{Q}$ with the mean of $\one_{\mathcal{Q}}\tr \bar{p}\opt$. 
	\end{proof}

	\subsection{Ambiguity Sets for Monotone Value Functions} \label{sec:improved_bounds}
	
	A significant limitation of the result in \cref{lem:single_one} is that the $\psi$ depends linearly on the number of states. We now explore an assumption that can alleviate this important drawback when the value functions are guaranteed to be monotone. In particular, the monotonicity assumption states that the value functions $v$ of the optimal robust policy must be non-decreasing in some arbitrary order which is know ahead of time. Assume, therefore, without loss of generality that:
	\begin{equation} \label{eq:value_property}
	v_1 \ge v_2 \ge \ldots \ge v_n  ~.
	\end{equation}
	
	Admittedly, monotonicity is a restrictive assumption, but we explore it in order to understand the greatest possible gains from tightening the known concentration inequalities. Yet, monotonicity of this type occurs in some problems, such as inventory management in which the value does not decrease with increasing inventory levels or medical problems in which the value does not increase with a deteriorating health state.
	
	It is important to note that any MDP algorithm that relies on the assumption \eqref{eq:value_property} needs to also enforce it. That means, the algorithm must prevent generating value functions that violate the monotonicity assumption. Practically, this could be achieved by representing the value function as a linear combination of monotone features.
	
	The bound \cref{lem:single_one} is large  because of the term $2^S$ which derives from the use of a union bound. The union bound is used because the $L_1$ norm can be represented as a maximum over an exponentially many linear functions:
	\[
	\| x \|_1 = \max_{\mathcal{Q} \subseteq \mathcal{I}} \left(\one_{\mathcal{Q}} - \one_{\mathcal{I} \setminus \mathcal{Q}} \right)\tr x ~. 
	\]
	Here, the set $\mathcal{I} = 2^S$ represents all indexes of $x$ and $\one_{\mathcal{Q}}$ is a vector that is one for all elements of $\mathcal{Q}$ and zero otherwise. We now show that under the monotonicity property \eqref{eq:value_property}, the $L_1$ norm can be represented as a maximum over a \emph{linear} (in states) number of linear functions. In particular, the worst-case optimization problem of the nature:
	\begin{equation} \label{eq:main_optimization}
	\begin{mprog}
	\minimize{p} v\tr p
	\stc \left(\one_{\mathcal{Q}} - \one_{\mathcal{I} \setminus \mathcal{Q}} \right)\tr (p - \bar{p})\le \psi, \quad \forall \mathcal{Q} \subseteq \mathcal{I}
	\cs \one\tr p = 1, 
	\cs p\ge 0
	\end{mprog} \end{equation}
	can be replaced by the following optimization problem:
	\begin{equation} \label{eq:simplified_optimization}
	\begin{mprog}
	\minimize{p} v\tr p
	\stc (\one_{k\ldots n} - \one_{1\ldots(k-1)}) \tr (p - \bar{p}) \le \psi, \quad  \forall k = 0,\ldots, (n+1)
	\cs \one\tr p = 1, 
	\cs p\ge 0
	\end{mprog} \end{equation}
	
	\begin{lemma} \label{lem:reduced_monotone}
		Suppose that \eqref{eq:value_property} is satisfied. Then the optimal objective values of \eqref{eq:main_optimization} and \eqref{eq:simplified_optimization} coincide.
	\end{lemma} 
	\begin{proof}
		Let $f^a$ be the optimal objective of \eqref{eq:main_optimization} and let $f^b$ be the optimal objective of \eqref{eq:simplified_optimization}. The inequality $f^a \ge f^b$ can be shown readily since \eqref{eq:simplified_optimization} only relaxes some of the constraints of \eqref{eq:main_optimization}. 
		
		It remains to show that $f^a \le f^b$. To show the inequality by contradiction, assume that each optimal solution $p^b$ to \eqref{eq:simplified_optimization} is infeasible in \eqref{eq:main_optimization} (otherwise $f^a \le f^b$). Let the constraint violated by $p^b$ be:
		\[ \left(\one_{\mathcal{C}} - \one_{2^S \setminus \mathcal{C}} \right)\tr (p - \bar{p})\le \psi, \]
		for some set $\mathcal{C}$. Since this constraint is not present in \eqref{eq:simplified_optimization}, that means that there exist $i$ and $j$ such that $i < j$, $i \in \mathcal{C}$, $j \notin \mathcal{C}$, and because the constraint is violated:
		\[ p_i^b = \bar{p}_i - \epsilon, \quad \text{or} \quad p_j^b = \bar{p}_j + \epsilon  \] 
		for some $\epsilon > 0$. Assume now that $p_i^b = \bar{p}_i - \epsilon$, the case when $p_j^b = \bar{p}_j + \epsilon$ follows similarly. 
		
		Now, choose the largest $k > i$ possible, and let $p^a = p^b$, with the exception of:
		\[ p_i^a = p_i^b + \epsilon, \quad \text{and} \quad p_k^a = p_k^b - \epsilon ~.\] 
		This does not increase the violation of the constraint by $p^a$ over $p^b$:
		\[ \left(\one_{\mathcal{C}} - \one_{2^S \setminus \mathcal{C}} \right)\tr (p^a - \bar{p})\le \left(\one_{\mathcal{C}} - \one_{2^S \setminus \mathcal{C}} \right)\tr (p^b - \bar{p}), \]
		And it does not increase the objective function:
		\[ v\tr p^a = v\tr p^b - \epsilon (v_i - v_j) \le v\tr p^b, \]
		and thus remains optimal in \eqref{eq:simplified_optimization}. Repeating these steps until no constraints are violated leads to a contradiction with the lack of an optimal solution to \eqref{eq:simplified_optimization} that is not optimal in \eqref{eq:main_optimization}.
	\end{proof}
	
	\Cref{lem:reduced_monotone} shows that we can replace the $L_1$ ambiguity set in \eqref{eq:ambiguity_l1} by the following set without affecting the solution.
	\begin{equation} \label{eq:simplified_l1_ambiguity}
	\aset_{s,a} = \{ p \in \Delta^\statecount \ss  
	(\one_{k\ldots n} - \one_{1\ldots(k-1)}) \tr (p - \bar{p}_{s,a}) \le \psi_{s,a}, \quad  \forall k = 0,\ldots, (n+1) \}
	\end{equation}
	
	Now, following the same steps as the proof of \cref{lem:single_one} but using \eqref{eq:simplified_l1_ambiguity} in place of \eqref{eq:ambiguity_l1} gives us the following result.
	\begin{lemma}[$L_1$ Error bound] \label{lem:single_one_monotone}
		Suppose that $\bar{p}_{s,a}$ is the empirical estimate of the transition probability obtained from $n_{s,a}$ samples for each $s\in\states$ and $a\in\actions$. Then:
		\[ \P \left[ \| \bar{p}_{s,a} - p\opt_{s,a}\|_1 \ge \psi_{s,a} \, \right] \le S \exp\left(-\frac{\psi_{s,a}^2 n_{s,a}}{2}\right) ~.\]
		Therefore, for any $\delta \in [0,1]$:
		\[ \P\left[ \| \bar{p}_{s,a} - p\opt_{s,a}\|_1 \le \sqrt{\frac{2}{n_{s,a}}  \log \frac{S^2 A}{\delta} } \, \right] \le 1 - \delta / (S A)  ~.  \]
	\end{lemma}
	
	\section{Detailed Descriptions of Selected Algorithms}
	
	\subsection{Computing Bayesian Credible Region}
	
	\begin{algorithm}[H]
		\KwIn{Distribution $\theta$ over $p\opt_{s,a}$, confidence level $\delta$, sample count $m$}
		\KwOut{Nominal point $\bar{p}_{s,a}$ and $L_1$ norm size $\psi_{s,a}$}
		Sample $X_1, \ldots, X_m \in \Delta^S$ from $\theta$: $X_i \sim \theta $\;
		Nominal point: $\bar{p}_{s,a} \gets (1/ m) \sum_{i=1}^m X_i $\;
		Compute distances $d_i \gets \lVert \bar{p}_{s,a} - X_i \rVert_1$ and sort \emph{increasingly}\;
		Norm size: $\psi_{s,a} \gets d_{(1-\delta)\,m}$\;
		\Return{$\bar{p}_{s,a}$ and $\psi_{s,a}$}\;
		\caption{Bayesian Credible Interval (BCI)} \label{alg:bayes}
	\end{algorithm}
	
\end{document}